\documentclass[11pt,twoside]{article} 
\usepackage{geometry}
\newcommand*\samethanks[1][\value{footnote}]{\footnotemark[#1]}
\author{Bhagyashree Puranik\thanks{Equal contribution. This work was supported by the National Science Foundation under Grants 1909320, 2224263, 2342253 and 2236483. This manuscript has been accepted for publication in the \textit{IEEE Journal on Selected Areas in Information Theory} special issue on information-theoretic methods for reliable and trustworthy ML ~\href{https://doi.org/10.1109/JSAIT.2024.3416078}{https://doi.org/10.1109/JSAIT.2024.3416078}.}\hspace{0.1cm},  Ozgur Guldogan\samethanks\hspace{0.1cm}, Upamanyu Madhow, Ramtin Pedarsani
\vspace{8pt} 
\\
University of California, Santa Barbara, USA 
}

\title{Long-Term Fairness in Sequential Multi-Agent Selection with Positive Reinforcement}

\usepackage[round]{natbib}

\usepackage{url}

\usepackage{amsmath,amssymb,amsthm,amsfonts,amscd,mathrsfs,mathtools}
\usepackage{algorithmic}
\usepackage{graphicx}
\usepackage{textcomp}

\def\BibTeX{{\rm B\kern-.05em{\sc i\kern-.025em b}\kern-.08em
    T\kern-.1667em\lower.7ex\hbox{E}\kern-.125emX}}

\usepackage{graphicx}
\usepackage{amsmath}
\usepackage{amsthm}
\usepackage{amsfonts}
\usepackage[labelformat=simple]{subcaption}

\usepackage{wrapfig}
\usepackage{xcolor}
\usepackage{hyperref}

\allowdisplaybreaks

\definecolor{ruddy}{rgb}{1.0, 0.0, 0.16}
\definecolor{gblue}{RGB}{29, 144, 255}
\definecolor{royalblue}{rgb}{0.25, 0.41, 0.88}
\hypersetup{
  colorlinks=true,
  citecolor=royalblue,
  linkcolor=ruddy,
  urlcolor=royalblue}
  
\DeclareMathOperator*{\argmax}{arg\,max}
\DeclareMathOperator*{\argmin}{arg\,min}

\newtheorem{theorem}{Theorem}
\newtheorem*{remark}{Remark}
\newtheorem{lemma}{Lemma}
\newtheorem{proposition}{Proposition}
\newtheorem{assumption}{Assumption}

\newcommand{\diff}[3][]{\frac{\mathrm{d}^{#1} #2}{\mathrm{d} #3^{#1}}}

\begin{document}

\date{}

\maketitle
\begin{abstract}
While much of the rapidly growing literature on fair decision-making focuses on metrics for one-shot decisions, recent work has raised the intriguing possibility of designing sequential decision-making to positively impact long-term social fairness. In selection processes such as college admissions or hiring, biasing slightly towards applicants from under-represented groups is hypothesized to provide positive feedback that increases the pool of under-represented applicants in future selection rounds, thus enhancing fairness in the long term. In this paper, we examine this hypothesis and its consequences in a setting in which \emph{multiple} agents are selecting from a common pool of applicants. We propose the \emph{Multi-agent Fair-Greedy} policy, that balances greedy score maximization and fairness. Under this policy, we prove that the resource pool and the admissions converge to a long-term fairness target set by the agents when the score distributions across the groups in the population are identical. We provide empirical evidence of existence of equilibria under non-identical score distributions through synthetic and adapted real-world datasets. We then sound a cautionary note for more complex applicant pool evolution models, under which uncoordinated behavior by the agents can cause \emph{negative} reinforcement, leading to a reduction in the fraction of under-represented applicants. Our results indicate that, while positive reinforcement is a promising mechanism for long-term fairness, policies must be designed carefully to be robust to variations in the evolution model, with a number of open issues that remain to be explored by algorithm designers, social scientists, and policymakers.
\end{abstract}

\section{Introduction}

With the increasing use of machine learning models for decision-making systems with significant societal impact, such as recruitment~\cite{li_hiring}, criminal justice~\cite{dressel2018}, and credit lending~\cite{berkovec2018}, there is also growing concern that such models may inherit existing bias in data, perpetuating and potentially exacerbating discrimination against certain groups in the population. This has prompted a growing body of research focused on developing fair and unbiased models, with much of the early literature focused on imposing notions of statistical fairness, such as equal selection rates or equal true positive rates, in static frameworks through pre-processing~\cite{zemel2013learning,gordaliza2019obtaining,kamiran2012data}, in-processing~\cite{zhang2018mitigating,zafar2017fairness}, or post-processing~\cite{hardt2016equality} mechanisms. 

Going beyond concepts of one-shot fairness and unbiasedness, there is also increasing interest in the long-term impacts of automated decisions under various dynamical models. For example, the potentially adverse consequences of myopic fairness are pointed out in~\cite{liu2018delayed}: unanticipated feedback dynamics due to the decisions made may change population statistics in an undesirable manner. For example, credit lending decisions which equalize true positive rates across groups to satisfy statistical fairness might lead to loans being offered to less creditworthy applicants from a disadvantaged group.  Lower repayment rates from this group may then end up further decreasing its creditworthiness. On the other hand, feedback dynamics could also be used to shape population statistics in a desired direction. We take a step in this direction by hypothesizing that biasing slightly in favor of an underrepresented group in a selection problem (e.g., hiring or college admissions) provides positive reinforcement, increasing the proportion of applicants from that group in future selection rounds. 

While most prior works on exploring such dynamics consider sequential decision-making by a single agent, we take a first step towards exploring the dynamics and long-term impact of multiple decision-making agents competing for a common pool of resources. We consider a selection problem, where we wish to use positive feedback to increase the proportion of underrepresented applicants in the presence of multiple agents (e.g, universities or companies) selecting from among a common pool of applicants. In order to focus attention on feedback dynamics, we consider a simplified model in which the agents are strictly rank ordered in terms of desirability from the applicants' point of view. For this model, we provide a mathematical framework for studying whether some level of cooperation between the agents helps promote long-term fairness, and whether the concept of positive reinforcement is robust to variations in the evolution model. We propose the evolution model of \textit{positive reinforcement} where if a higher proportion of applicants from a specific group is selected, it will lead to an increase in the proportion of applicants from that group in subsequent rounds. We study variants of this basic model of reinforcement, a particularly interesting one being the notion of {\it role model reinforcement}. Here only a certain fraction of the admitted applicants (as opposed to all admitted applicants) are in a position to influence applicant proportions in the future, by virtue of being role models. The idea that role models in society, to which a group can relate, could positively influence more aspirants to enter a field is supported by several works in social sciences and economics~\cite{morgenroth2015motivational, bettinger2005faculty}.

\paragraph*{Contributions}
The contributions of this work are summarized as follows: 
\begin{itemize}
    \item We propose the Multi-agent Fair-Greedy (MFG) policy, in which agents operate in a decentralized fashion, maximizing a greedy utility (based on the scores of selected applicants) while minimizing the disparity from a fixed long-term fairness target (based on the deviation of the selected proportion of minority applicants from a target proportion deemed to be socially fair). We characterize optimal actions under this policy and theoretically demonstrate the convergence of the overall applicant pool and the admission proportions to the desired long-term fairness target under the \textit{pure positive reinforcement model} for the evolution of the composition of the applicant pool.
    \item Different population dynamics are studied empirically to evaluate the impact of variations in population behavior. We find that while the decentralized MFG policy attains long-term fairness under pure positive reinforcement, if the population dynamics follow a simple variant, termed the \textit{role model reinforcement}, uncoordinated behavior by agents can result in an overall {\it negative feedback,} leading to a steady decrease in the number of underrepresented applicants in the pool. We propose a centralized version of the MFG policy which can restore positive feedback.  
    \item We illustrate our mathematical framework through comprehensive experimental results based on synthetic and semi-synthetic datasets, highlighting variations in system behavior under different evolution models, and under decentralized and centralized MFG policies. These results show that positive reinforcement indeed has potential for promoting long-term fairness even with multiple agents, but that policy design must be carefully considered in order to be robust to changes in the evolution model. 
\end{itemize}

\section{Related Work}
There is a rich and rapidly growing literature on \textit{fair} strategies that mitigate bias in one-shot algorithmic decision making, including pre-processing the labels or data and reweighting costs based on groups~\cite{kamiran2012data}, reducing mutual information between sensitive attributes and predictions~\cite{cho2020fair,kamishima2012fairness}, adversarial de-biasing~\cite{zhang2018mitigating}, addition of constraints that satisfy fairness criteria~\cite{zafar2017fairness}, learning representations that obfuscate group information~\cite{zemel2013learning}, and other information theoretic methods~\cite{calmon2017optimized, ghassami2018fairness, kenfack2023learning, kairouz2022generating,shen2023fair}. However, recent research has introduced an intriguing possibility: the design of sequential decision-making strategies that can have a positive long-term impact on social fairness and a study of their consequence on the population.

The effects of fairness-aware decisions on underlying population statistics were first studied in~\cite{liu2018delayed} for two-stage models: in the first stage, the algorithmic decisions are designed such that fairness constraints such as statistical parity or equal opportunity are satisfied, while the second stage examines the impact of these fairness interventions on the groups. Each sample is associated with a score corresponding to the probability of a positive outcome, sampled from group-specific distributions. The policy-maker or institution chooses selection policies to maximize utility subject to statistical fairness. The measure of interest in~\cite{liu2018delayed} is the expected change in the mean of the score distributions for the two groups as a result of one step of feedback. It was found that imposing equal selection rates or true positives could lead to either improvement or cause harm, particularly to the minority group, depending on certain regimes. The results of~\cite{liu2018delayed} highlight the importance of going beyond static notions of fairness in algorithm/policy design.

Several works propose notions of statistical fairness focused on improving feature distributions among groups. For example,~\cite{heidari2019long} presents a fairness notion that equalizes the maximum change in reward for groups with the same effort budget for improving their features. The authors examine how fairness interventions impact evenness, centralization, and clustering in the groups through their efforts, affecting score distributions. In another example,~\cite{guldogan2023equal} proposes to equalize the proportion of unqualified candidates from different groups that can be qualified with a limited effort for improving their features. The authors investigate how statistical fairness notions change feature distributions among groups in the long term through modeling feature evolution.

Formal investigation of temporal effects of decision feedback and their equilibrium is typically performed in reinforcement learning~\cite{jabbari2017fairness, wen2021algorithms, heidari2018preventing,yin2023longterm} or bandit settings~\cite{joseph2018meritocratic,chen2020fair, gillen2018online,patil2020achieving,ghalme2021long}. The environment is described through a Markov Decision Process (MDP) framework where at each time, the decision-maker in a particular state takes an action and receives a reward. State transitions are governed by update models, and fairness constraints are included within reward definitions.

The long-term effects of fair decisions on the qualification rate of the group, which is defined as the probability of an individual from a particular group being qualified, is investigated in~\cite{zhang2020fair} under a partially observable MDP setting. The group proportions over time are fixed, but the decisions affect the feature distributions which in turn change their true qualification state, which is modeled as a hidden state. Decisions are performed on the features to maximize myopic instantaneous utility, subject to statistical fairness constraints. Threshold-based policies and their equilibrium are studied under two regimes: ``lack of motivation,'' where the probability of remaining qualified on receiving a positive decision is less than that on being rejected, and ``leg-up,'' which is the opposite, where an accepted individual becomes inspired to become more qualified. Studies such as ~\cite{mouzannar2019fair, williams2019dynamic} also examine the effects of fair policies on the distributions of the features. In particular,~\cite{mouzannar2019fair} studies how the qualification profiles of groups are influenced by a policy that imposes demographic parity (equal selection rates) across two groups of the population. They assume that social equity is achieved through equalized qualification profiles in equilibrium. The dynamic model in \cite{williams2019dynamic} is motivated by credit lending. The authors model the distributions of loan repayment likelihood (payback probabilities) by group, and examine the dynamics governed by the hypothesis that granting loans leads to upward mobility for the population if they are repaid. They study the impact of fair decisions on loan repayment likelihood and the negative effects of unequal misestimation of payback probabilities across groups, even if the decisions are fair. In contrast to studying the effects on group qualification, it is also imperative to understand how group representation could vary over time. It was first shown in~\cite{hashimoto} that empirical risk minimization can exacerbate the disparity in group representation. In the context of long-term fairness,~\cite{zhang2019group} study how algorithmic decisions which are constrained by statistical fairness could degrade the representation of a minority group, and eventually cause the loss of minority representation in the system.

\textit{Our prior work:} The work reported here builds upon preliminary results shown in our conference paper~\cite{puranik2022dynamic}, where we first introduce the problem of selecting applicants from a pool under the setting of a single institution. We introduce the {positive reinforcement model} for evolution of the applicant pool, which is similar in spirit to the ``leg-up'' model in~\cite{zhang2020fair}, except that it applies at the level of a population rather than an individual: selection of a larger proportion of applicants from a specific group feeds back into society and leads to an increase in the proportion of applicants from that group in future rounds. The score of an applicant, drawn from a group-specific distribution, is taken to represent the level of qualification of that individual. The Fair-Greedy policy, proposed in a single institution setting, greedily balances the sum of the scores of selected applicants against a fairness measure which increases with the deviation of the selected proportion of applicants from the under-represented group from a specified long-term target proportion. Theoretical results for identical score distributions across groups, and empirical results for other models, show convergence to a long-term fairness target set by the decision-maker. While the preliminary results in~\cite{puranik2022dynamic} show long-term fairness through positive feedback under a single institution, in this journal paper, we extend the promise of positive reinforcement to multiple-agents, consider variations in the evolution model, and find that careful design of feedback mechanism is required.

All of the preceding works on long-term fairness focus on settings with a single decision-maker accessing a pool of samples from the population. However, in real-world selection processes such as college admissions or hiring, there are often multiple agents or institutions competing for a common pool of applicants. The composition of the applicant pool can therefore be affected by the collective decisions of these agents. To the best of our knowledge, ours is the first work to study the long-term evolution of fairness in such a multi-agent setting. While two-sided matching problems, where multiple agents and resources are to be matched to each other, are well studied in resource allocation problems and game theoretic settings~\cite{dai2021learning, cho2022two}, the notion of long-term fairness in such settings is yet to be examined, and the definitions of fairness considered are quite distinct from those in fairness literature. Since our focus is on dynamical evolution, we sidestep the problem of matching between institutions (agents) and applicants (resources) by assuming that each applicant has the same fixed preference among the institutions.

\section{Modeling Multi-Agent Decision-Making}
Consider the problem of $K$ agents, or institutions, selecting from a common pool of resources, or applicants. The applicant pool is composed of applicants from two groups $g = \{0,1\}$, based on a sensitive attribute. Without loss of generality, we usually refer to $g=0$ as the minority group and $g=1$ as the majority group. We denote the number of applicants in round $t$, belonging to group $g$ by $N_t^g$, and the total number of applicants as $N_t^0 + N_t^1 = N_t$. Every applicant is associated with a score which is sampled from a group-dependent distribution $\mathcal{P}^g$. 

In many applications, the institutions have a growth in the program size, which is typically proportional to the population. For instance, the University of California system intends to serve top one-eighth of California's high school graduating class population~\cite{uc2024plan}. Motivated by this line of thought, we consider that every institution fixes the total number of applicants it can admit, through the {\it capacity} $c_k$ for institution $k$, set as a fraction of applicants it admits, out of the total number of applicants in the pool in that round. Thus we have $\sum_{k = 1}^{K} c_k < 1$. 

Each institution has to determine how to fill its admission slots, particularly by deciding how many minority and majority applicants are to be admitted, while satisfying two-fold objectives: (i) to accept applicants with the largest scores (ii) to achieve a long-term fairness target (denoted by $\alpha \in [0,1]$) in the admission of applicants, which measures the proportion of admitted applicants belonging to the minority group, $g=0$. The goal of the $K$ institutions is to move towards the long-term fairness target, while admitting applicants with the highest scores. We denote the number of applicants admitted by institution $k$ as $A_{k,t} = c_k N_t$, out of which $A_{k,t}^g$ belong to group $g$. Thus, we have $A_{k,t} =A_{k,t}^0 + A_{k,t}^1$, and the institutions hope to make decisions such that $A_{k,t}^0/A_{k,t}$ approaches $\alpha$ in equilibrium for all $k$.

We assume that the institutions are ranked, and that every applicant prefers institutions in the same order. Without loss of generality, we assume that the institutions are indexed by rank order, with $k=1$ corresponding to the highest-ranked institution. Since the preference order is fixed, the matching of applicants to institutions simplifies to institutions picking from the application pool sequentially, in their ranked order. The decision variable at each institution in each round is the fraction of admitted applicants from each group (selecting from the highest scoring applicants in each group, chosen from among the pool that remains after higher-ranked institutions have made their selections). We, therefore, have an effectively sequential selection process in each round, where the selections are started by the highest-ranked institution and completed by the lowest-ranked institution.

The problem is formulated as a Markov Decision Process (MDP). The state $s_t \in [0,1]$ is the fraction of applicants from group $g = 0$ in the applicant pool. Thus, at round $t$ we have $s_t = N^0_t/N_t$. An institution $k$ is associated with its respective action $a_t^k \in [0,1]$, which is the fraction of applicants admitted from group $g = 0$ among the total number of applicants the institution $k$ admits. Thus, $a_t^k = A^0_{k,t}/A_{k,t}$. The actions or the \textit{decision variables} of each of the institutions are to be determined based on optimizing their respective \textit{fairness-aware utilities}, defined below.

First, a set of $N_t$ applicants in round $t$ are associated with a set of scores $\{X^g_i\}$, where each $X^g_i$, $i \in [N^g_t$],  is sampled independently from $\mathcal{P}^g$. Additionally, $X^g_{(i)}$ denotes the $i^{th}$ top score out of $N^g_t$ applicants from group $g$. We define the score-based reward for an institution with rank $k$ as the expectation of the sum of scores of the applicants admitted by the institution normalized by the number of total applicants it admits. Note that for an institution $k$, this depends upon the actions of the preceding $[1, k-1]$-ranked institutions, as each institution selects sequentially from the remnant pool of applicants. For $k \in \{2,3\ldots,K\}$, we define $m_{k,t}^g$  as the number of applicants from group $g \in \{ 0,1 \}$ that have already been selected by higher ranked institutions $1,...,k-1$, with ${m}^{g}_{1,t} = 0$. Thus, we have the score-based reward:
\begin{equation}
    R_{k}(s_t, a_{t}^k)
    = \frac{1}{A_{k,t}} \mathbb{E}\left[\sum_{i = {m}^{0}_{k,t}+ 1}^{{m}^{0}_{k,t}+A^0_{k,t}} X^0_{(i)} + \sum_{i = {m}^{1}_{k,t}+1}^{{m}^{1}_{k,t}+A^1_{k,t}} X^1_{(i)} \right] =\frac{1}{A_{k,t}} \mathbb{E}\left[\sum_{i = {m}^{0}_{k,t} +1}^{{m}^{0}_{k,t}+ a_{t}^k A_{k,t}} X^0_{(i)} + \sum_{i = {m}^{1}_{k,t}+ 1}^{{m}^{1}_{k,t}+(1 - a_{t}^k)A_{k,t}} X^1_{(i)} \right]
    \label{eqn:score_reward}  
\end{equation}
with the group-specific offsets given by
\begin{equation}
\textit{m}^{0}_{k,t}  = \sum_{j = 1}^{k-1} c_j N_t a_t^{j} \text{ and }  \textit{m}^{1}_{k,t}  = \sum_{j = 1}^{k-1} c_j N_t (1-a_t^{j}).
\end{equation}
Each institution is associated with its fairness loss, which is the squared difference between the proportion of group $g=0$ out of the total admitted, and the long-term fairness target $\alpha$, given by $L_{k}(a_t^k) = (a_t^k - \alpha)^2.$ The considered fairness loss is a simple measure of disparity from the long-term fairness target, and penalizes the bias against both groups on the long run.

The ``multi-agent fairness aware'' utility of each institution is the sum of its {\it score-based reward} and {\it fairness loss}, expressed as 
\begin{equation}
    U_k(s_t, a_{t}^k) = R_{k}(s_t, a_{t}^k) - \lambda L_{k}(a_{t}^k).\label{eqn:utility}
\end{equation}
where $\lambda\geq 0$ is a parameter that governs balancing fairness objective with score-based reward.

Thus each institution aims to maximize its score-based reward by admitting applicants with the highest scores, while minimizing its disparity from the long-term fairness target.  

\noindent
\paragraph*{Multi-agent Fair-Greedy Policy}
We propose the {\it Multi-agent Fair-Greedy (MFG) policy}, which is a set of policies $\pi_{MFG}(s_t) =\{\pi^{1}_t, \pi^{2}_t\, \ldots, \pi^{K}_t\}$, where each institution optimizes its own multi-agent fairness aware utility as follows:
\begin{equation}
    \pi^{k}_t  = \argmax_{a_{t}^k\in\mathcal{A}_{t}^k} U_k(s_t, a_{t}^k).
\end{equation}
where $\mathcal{A}_{t}^k$ is the set of feasible actions for institution $k$ in round $t$, which depends on the actions of higher ranked institutions and the state $s_t$. After higher ranked institutions have made their selections, the feasible action space for institution $k$ is determined by the remaining capacity of the institution and the remaining applicants in the pool, and can be written as $\mathcal{A}_{t}^k =  [\max(0,1-\frac{1-s_t-\sum_{j=1}^{k-1}(1-\pi^{j}_t)c_j}{c_k}), \min(1,\frac{s_t-\sum_{j=1}^{k-1}\pi^{j}_t c_j}{c_k})]$. For each institution, optimization of the instantaneous multi-agent fairness aware utility is equivalent to finding group-specific set of thresholds on the score distributions.  In our rank-ordered model for institutions, each institution admits applicants whose scores exceed its threshold, choosing from the pool remaining after higher-ranked institutions have made their selections. However, we reiterate that the state $s_t$ represents the original minority group proportion in the pool at round $t$.

\noindent\paragraph*{Resource pool evolution} A key ingredient of our formulation is modeling the manner in which the collective decisions of the institutions affect the pool of resources seen in the subsequent rounds of the problem. This is a complex relationship requiring data collected from carefully designed long-term experiments. While such data is not yet available, we can develop valuable insights by exploring different models for how the applicant pool might be shaped by institutional policies.

We assume that the state $s_t$ is a random variable with mean $\theta_t$ and bounded variance. We model the number of applicants from group $g=0$ and $g=1$ as being sampled from Poisson distributions, with $N^0_t \sim\text{Poisson}(\theta_t N)$ and $N^1_t \sim \text{Poisson}((1-\theta_t) N)$, where $N$ is the expected number of applicants in the pool. The total number of applicants in round $t$ is given by $N_t = N^0_t + N^1_t$.\footnote{When both $N^0_t$ and $N^1_t$ are equal to zero, the round is completed without any admissions.} The choice of the Poisson distribution is for the sake of simplicity, and the outcomes of our analysis remain independent of this particular choice of distribution. We consider different models for the evolution of $\theta_t$ in this paper, focusing first on the model of {\it pure positive reinforcement}. In this model, a higher proportion of admission of a particular group in comparison to its application proportion, motivates more applicants from the group to participate in future selection rounds. Here the evolution follows: 
\begin{equation}
    \theta_{t+1} = \left[ \theta_t + \eta_t (\pi^{W}_{t} - s_t) \right]_{\mathcal{C}}
    \label{eqn:pure_pos_reinf}
\end{equation}
where $[ \ ]_\mathcal{C}$ is the projection onto the set $\mathcal{C} = [0+\epsilon, 1- \epsilon]$ (or, simply clipping the mean parameter) where $\epsilon$ is a small positive number to avoid the mean parameter from reaching the boundary of the set and it implies that the any group cannot be completely eliminated from the pool. Here, $\eta_t$ is a step size parameter and $\pi^{W}_t$ represents the weighted actions of all the institutions.

In particular, under the MFG policy, $\pi^W_t$ depends upon on the policies of all the institutions weighted by their capacities, which is also equivalent to the proportion of minority group among all the admitted applicants, defined as:
\begin{equation}
    \pi^W_t = \frac{\sum_{k = 1}^{K} c_k \pi^{k}_t}{\sum_{k = 1}^{K} c_k}.
    \label{eqn:pi_W}
\end{equation}
Thus, the collective decisions of all institutions, in particular their admission of applicants from a particular group, promote more such applicants to participate in the process, and shape the evolution dynamics.

In the remainder of this section, we characterize the MFG policy in a decentralized setting, under the pure positive reinforcement model. We show that the optimal policy for each institution depends on actions that optimize the score-based reward alone, and the fairness loss alone. We therefore first derive the reward-optimal and fairness-optimal actions for each institution, and then characterize the MFG policy $\pi^k_t, k=1,...,K$ in terms of these actions. We then show that, for pure positive reinforcement and identical score distributions $\mathcal{P}^{g}$, the applicant pool (state) and the admission proportions (actions) for all institutions converge to the long-term fairness target. We use the following assumptions in our convergence analysis:
\begin{assumption}\label{asm:large}
The expected number of applicants in the pool $N$ is large enough that the empirical distribution of scores can be replaced by statistical distributions using the law of large numbers.
\end{assumption}
This assumption is required for the asymptotic analysis of the applicant pool and admission proportions. It is also used to derive the reward-optimal actions for the institutions.
\begin{assumption}\label{asm:identical}
The score distributions $\mathcal{P}^{g}$ are identical across groups.
\end{assumption}
This assumption implies that the only difference in the groups is the proportion of applicants from each group in the pool at each round. We relax this assumption in Section~\ref{sec:exps_synt} and provide empirical evidence for the existence of equilibria of the applicant pool proportion and the admission proportions.

\paragraph*{Fairness-optimal action} Each institution can minimize its fairness loss by setting its action to be equal to the long-term fairness target:
\begin{equation}
    a^{k}_{F,t} = \argmin_{a_t\in\mathcal{A}_{t}^k} L_{k}(a_t) = [\alpha]_{\mathcal{A}_{t}^k}, \quad \forall k \in \{1,2,\ldots, K\}. \label{eqn:opt_fair_action}
\end{equation}
where $[\alpha]_{\mathcal{A}_{t}^k}$ is the projection of $\alpha$ onto the feasible action space $\mathcal{A}_{t}^k$.

The next theorem characterizes the reward-optimal actions for the institutions.

\begin{theorem}
\label{thm:opt_score_act}
Under Assumptions~\ref{asm:large} and~\ref{asm:identical}, the score-based reward function, $R_{k}(s_t, a_t^k)$, is concave and the reward-optimal action for an institution $k$ is given by
\begin{equation}
    a^{k}_{S,t} = \argmax_{a_t^k\in \mathcal{A}_{t}^k} R_{k}(s_t, a_t^k)  =  \left[ s_t + \frac{1}{c_k}\left( \sum_{j = 1}^{k-1}c_j (s_t - \pi^{j}_t )\right)\right]_{\mathcal{A}_{t}^k}, \quad \forall k \in \{2,3,\ldots, K\}
    \label{eqn:opt_score_action}
\end{equation}
and $a^{1}_{S,t} = s_t$ where $[\cdot]_{\mathcal{A}_{t}^k}$ denotes the projection onto the feasible action space $\mathcal{A}_{t}^k$.
\end{theorem}
The detailed proof of this theorem can be found in Appendix~\ref{sec:app_opt_greedy_action}.

\begin{remark}
Under the asymptotic regime of Assumption \ref{asm:large} (i.e., assuming that empirical distributions of applicant scores can be replaced by statistical distributions), the proof of Theorem \ref{thm:opt_score_act} can rely on the following simplification. For institution $k$, the selection of applicants with top scores is equivalent to thresholding the group’s score, and admitting applicants with scores within the group-specific lower and upper thresholds, denoted by $t_g^{\text{k,S,low}}$ and $t_g^{\text{k,S,up}}$ respectively. The upper threshold $t_g^{\text{k,S,up}}$ equals the lower threshold of institution $k-1$ for group $g$. In Lemma~\ref{lem:concave_score_reward} we show that for reward optimality for any institution $k$, its group-specific lower thresholds should be equal; $t_0^{\text{k,S,low}} = t_1^{\text{k,S,low}}$, if possible. While this holds even if the group-specific score distributions are different, we can derive closed-form expressions for reward-optimal actions as in Theorem~\ref{thm:opt_score_act} when the distributions are identical across groups.
\end{remark}

\begin{remark}
The optimal policy of institution $k$ that maximizes its multi-agent fairness aware utility can be expressed as a convex combination of its pre-projected optimal score-based reward action and optimal fair-only actions, as $\pi^{k}_{t} = \gamma_{k,t} \left(s_t + \frac{1}{c_k}\left( \sum_{j = 1}^{k-1}c_j (s_t - \pi^{j}_t )\right)\right) + (1-\gamma_{k,t}) \alpha$, where $\gamma_{k,t} \in [0,1]$ (please refer to Appendix~\ref{sec:app_weighted_MFG_expr} for details). The closed-form expression for $\gamma_{k,t}$ is not available in general, but it can be numerically computed due to the concavity of the fairness-aware utility function. It can be shown that the reward-optimal action can be expressed as:
\begin{equation}
       a^{k}_{S,t} = s_t + \frac{(s_t-\alpha)}{c_k}\left(\sum_{j = 1}^{k-1}c_j {\displaystyle \prod_{i = j}^{k-1}} (1 - \gamma_{i,t} )\right). \label{eqn:reward_optimal_action_v2}
\end{equation}
Thus, when the hyperparameter $\lambda>0$, if $s_t < \alpha$, for $k \in \{2,3,\ldots,K\}$ we have $a^{k}_{S,t} < s_t$, and vice-versa. Note that, for a special case where there is a single institution, $\pi^{1}_t$ lies between the state $s_t$ and long-term fairness target $\alpha$. But that is not necessarily the case for $\pi^{k}_t$ for $k>1$, when there are multiple institutions.
\end{remark}

In the following lemma, we characterize the weighted policy of the institutions, that ultimately governs the pool evolution \eqref{eqn:pure_pos_reinf}.
\begin{lemma}
\label{lem:weighted_policy}
Under Assumptions~\ref{asm:large} and~\ref{asm:identical}, the weighted MFG policy $\pi^W_t$ can be expressed as the following under the pure positive reinforcement model:
\begin{equation}
    \pi^W_t = s_t + (\alpha-s_t)\frac{\displaystyle\sum_{j = 1}^{K} c_j \displaystyle \prod_{i = j}^{K} (1 - \gamma_{i,t} )}{\displaystyle\sum_{j = 1}^{K} c_j } = s_t + (\alpha-s_t)\bar{\gamma}_{t},
\label{eqn:weighted_MFG_expr}
\end{equation}
where $\bar{\gamma}_{t}\in[0,1].$
\end{lemma}
Please refer to Appendix~\ref{sec:app_weighted_MFG_expr} for proof. 

We utilize this result to prove the convergence of the applicant pool and admission proportions to the long-term fairness target. In the following theorem, we first show that the target proportion is a unique fixed-point of the pool evolution update. We then show that the weighted policy lies between the applicant and target proportions, as a result of which the mean state parameter and the state converge to the target. Although our analysis is for identical group-wise score distributions, we later provide empirical evidence of equilibrium in state and actions of all institutions when the score distributions are different across groups.

\begin{theorem}
\label{thm:conv_mfg}
Under Assumptions~\ref{asm:large},~\ref{asm:identical} and $\lambda\neq0$, the weighted MFG policy is such that $\pi_{t}^W \in (s_t, \alpha)$ if $s_t < \alpha$, or $\pi_{t}^W \in (\alpha, s_t)$ if $s_t > \alpha$. Further, $\alpha$ is a unique fixed-point of the weighted policy $\pi_t^{W}$. In addition, the applicant pool proportion converges to the long-term fairness target $\alpha$, if the step size parameter $\eta_t$ is decaying with time and satisfies the assumptions that $\sum_t \eta_t =  \infty$ and $\sum_t \eta_t^2 < \infty$. Further, the admission proportions of all institutions approach the long-term fairness target at equilibrium.
\end{theorem}
The proof of this theorem can be found in Appendix~\ref{sec:app_weighted_MFG_expr}.
\begin{remark}
    For identical group-wise distributions as considered in Theorem \ref{thm:conv_mfg}, the convergence to the target $\alpha$ holds irrespective of the value of the hyperparameter $\lambda$ weighing fairness in the utility function, as long as it is positive ($\lambda>0$).
    For non-identical score distributions across groups, we provide empirical evidence for the existence of equilibria of the applicant pool proportion and the admission proportions in Section~\ref{sec:exps_synt}. However, the actual value of $\lambda>0$ is now found to be important in determining how close to the fairness target we get. 
\end{remark}

\noindent\paragraph*{What happens if we ignore fairness?} 

In the scenario of identical score distributions, if each institution completely disregards the fairness objective and optimizes only its score-based reward, as seen from \eqref{eqn:reward_optimal_action_v2} where $\gamma_{k,t} = 1$, it follows that $\pi^k_t = a^{k}_{S,t} = s_t$ for all $k$. Consequently, the mean parameter $\theta_t$ would not experience any drift, and the composition of the applicant pool would remain unchanged. While each institution greedily acquires the best applicants by sacrificing diversity, there is no hope of influencing the participation of the underrepresented. However, biasing slightly in favor of the underrepresented group by introducing the fairness objective could lead to both $\theta_t$ and $\pi^k_t$ converging to the long-term fairness target.

\section{Variations on the Pool Evolution Model}
Achieving long-term fairness relies on how the applicant pool evolves with institutional decisions over time. In this section, we introduce variations of the pure positive reinforcement model, and examine how the MFG policy fares under these new variants of the evolution model termed (i) order-based (ii) weighted (iii) role-model reinforcement. The bulk of our discussion in this section focuses on the motivation and study of role-model reinforcement and pitfalls under it, while we first briefly discuss the other two below: 

\paragraph*{Order-based positive reinforcement}
This simple variant allows control over the strength of positive reinforcement by raising the feedback in pure positive reinforcement to a power $\beta > 0$:
\begin{equation}
    \theta_{t+1} = \left[ \theta_t + \eta_t \text{sign}(\pi^{W}_{t} - s_t)|\pi^{W}_{t} - s_t|^\beta\right]_{\mathcal{C}}.
    \label{eqn:order_pos_reinf}
\end{equation}
Pure positive reinforcement is a special case of the above model for $\beta = 1$. The feedback is amplified for $\beta < 1$ and is attenuated for $\beta > 1$, since $0 \leq |\pi^{W}_{t} - s_t| \leq 1$.

\paragraph*{Weighted positive reinforcement}
The second simple variation allows custom weights $\{ z_k > 0 \}$ on the institutions' policies
\begin{equation}
    \pi^z_t = \frac{\sum_{k = 1}^{K} z_k \pi^{k}_t}{\sum_{k = 1}^{K} z_k},
    \label{eqn:weight_pol}
\end{equation}
with pool evolution
\begin{equation}
    \theta_{t+1} = \left[ \theta_t + \eta_t (\pi^{z}_{t} - s_t)\right]_{\mathcal{C}}.
    \label{eqn:weight_pos_reinf}
\end{equation}
Here $z_k = c_k$ reduces back to pure positive reinforcement, which gives greater influence to an agent with higher capacity. The more general models \eqref{eqn:weight_pol}-\eqref{eqn:weight_pos_reinf} here allow us more flexibility in assigning influence. For example, we can give equal influence to all institutions by setting $z_k = z > 0$.

\paragraph*{Role model reinforcement}
Thus far, we have assumed that all selected applicants in an institution have equal ability to influence the reinforcement. However, applicants who do well post-selection in a given institution, who we term {\it role models,} could influence future applicant pools significantly more than other admitted applicants in the institution. The motivation for this model stems from the fact that in practice, success in an institution depends not just on qualifications at the time of admission, but on support mechanisms within the institution to aid in continued growth and success of the admitted applicants, as well as on circumstances that may be difficult to characterize. However, to develop analytical insight into what happens when positive reinforcement occurs due to the applicants who do well post-selection, we consider a score-based criterion for identifying these role models: among all applicants (from both groups) admitted to institution $k$, the role models are composed of a fraction $r \in [0,1]$ of applicants with the highest scores.  From among this group, the fraction of minority (group $g=0$) applicants, denoted by $r^k_t$, drive the pool evolution. We now develop the notation required to define this model precisely.

Let $\mathcal{X}_t^k$ denote the set of scores of all applicants admitted by the institution $k$ at time $t$,
\begin{equation}
\mathcal{X}_t^k=\left\{X^0_{({m}^{0}_{k,t}+1)},X^0_{({m}^{0}_{k,t}+2)},\dots,X^0_{({m}^{0}_{k,t}+A^0_{k,t})}, X^1_{({m}^{1}_{k,t}+1)},X^1_{({m}^{1}_{k,t}+2)},\dots,X^1_{({m}^{1}_{k,t}+A^1_{k,t})}\right\}.
\end{equation}
The role models for each institution are selected as top $r$ proportion of the admissions with the highest scores, $\mathcal{R}_t^k={\rm argmax}_{\mathcal{X}'\subset \mathcal{X}_t^k, |\mathcal{X}'|=\lfloor rA_{k,t}\rfloor}\sum_{x\in \mathcal{X}'} x.$ Then, let $r_t^k$ denote the fraction of the minority group in the role models set $\mathcal{R}_t^k$ for the institution $k$ at round $t$. It can be written as:
\begin{equation}
    r_t^k = \frac{\left|\mathcal{R}_t^k\cap\left\{X^0_{({m}^{0}_{k,t}+1)},X^0_{({m}^{0}_{k,t}+2)},\dots,X^0_{({m}^{0}_{k,t}+A^0_{k,t})}\right\}\right|}{|\mathcal{R}_t^k|}.
    \label{eqn:role_model}
\end{equation}
Then, the pool evolution model depends on the fraction of the applicants from the minority group among the role models, $r_t^k$, instead of the fraction of the applicants from the minority in the admissions, $\pi^{k}_t$. The weighted role model parameter is defined as
\begin{equation}
    \pi^r_t = \frac{\sum_{k = 1}^{K} c_k r_t^k}{\sum_{k = 1}^{K} c_k}.
    \label{eqn:role_pol}
\end{equation}
If a group has a higher proportion of role models in comparison to its application proportion, it will provide positive reinforcement. The pool update is then governed by:
\begin{equation}
    \theta_{t+1} = \left[ \theta_t + \eta_t (\pi^{r}_{t} - s_t)\right]_{\mathcal{C}}.
    \label{eqn:role_pos_reinf}
\end{equation}
For $r=1$, role model reinforcement reduces back to pure positive reinforcement. Setting $r = 0.5$ implies that minority applicants whose scores are above the median score for admitted applicants in their institution are role models driving pool evolution.

While we showed that convergence to long-term fairness target is assured under pure positive reinforcement, we shall now see that it is not the case, under this variant evolution model. First, let us focus first on a special case when $K=1$. Under a limiting case of role-model reinforcement when $r=1$, convergence to long-term fairness is guaranteed. In the other extreme case, when the bar for role models is extremely high, i.e., $r=\epsilon$ (small), $\pi^r_t = r^k_t \approx s_t$. Thus, although the MFG policy biases in favor of minority group, there could be no drift in the mean state parameter, leading to a stagnation in the composition of the applicant pool. However, when there are multiple institutions involved, this could lead to explicit negative feedback, causing the proportion of minority in the applicant pool to steadily reduce. We now show in Proposition~\ref{prop:role_model} below that role model reinforcement can lead to negative feedback under the MFG policy, specifically when there are multiple institutions($K>1$). Assuming identical group-wise score distributions, we show that if the long-term fairness target is higher than the initial proportion of the minority group, the first institution will admit more minority applicants than their application proportion. This results in the subsequent institutions having a higher proportion of top admissions from the majority group, leading to a lower proportion of minority group role models. This can eventually cause the minority group to be removed from the applicant pool, as formally demonstrated next, and empirically supported by experiments in Section~\ref{sec:exps_synt}. Furthermore, we empirically demonstrate that coordinated behavior by the agents could help in alleviating this negative feedback.

\begin{proposition}
\label{prop:role_model}
Under Assumptions~\ref{asm:large} and~\ref{asm:identical}, for role model reinforcement there exists a role model parameter $r$ that is small enough such that the MFG policy can cause negative feedback leading to the loss of representation of the under-represented group in the pool.
\end{proposition}
Please see Appendix~\ref{sec:app_role_model_neg_fb} for proof.

\begin{remark}
The evolution of the applicant pool under role model reinforcement is such that only admitted applicants from group $g=0$ with scores larger than a {\it group-independent threshold} $t_r^k$, will contribute to reinforcing the pool. The threshold increases as we reduce the role model parameter $r$, raising the bar for being a role model. Let $t_g^{\text{k, low}}$ and $t_g^{\text{k, up}}$ denote the lower and upper thresholds of the MFG policy for group $g$ and institution $k$. The key idea behind the proof is that when the initial mean parameter $\theta_t$ is small, there exists a small enough $r$ such that for institution $k$, we have $t_r^k \geq t_1^{\text{k, low}} > t_0^{\text{k, low}}$. Using this, we show that the role model proportions are such that $r_t^k < s_t$ for $k \in \{2,3,\ldots, K\}$ and $r_t^1 = s_t$. Hence, the mean parameter obtains a negative drift, and the proportion of the minority group in the applicant pool approaches zero. 
\end{remark}

In order to mitigate the effects of such negative feedback, we define a centralized version of the MFG policy in the next section.

\subsection{Centralized Multi-agent Fair-Greedy Policy}
The selection process under the original MFG policy, in which the institutions do not cooperate, is sequential based on the ranking order of the institutions. We now consider decision-making by a central coordinator which knows the utility of each institution, and maximizes the sum utility across institutions:  
\begin{equation}
    U(s_t, a_{t}^1,a_{t}^2,\dots, a_{t}^K) = \sum_{k=1}^K U_k(s_t, a_{t}^k).
    \label{eqn:total_util}
\end{equation}

Thus, the {\it Centralized Multi-agent Fair-Greedy (CMFG) policy} is defined as a set of policies \(\pi_{CMFG}(s_t) =\{\pi^{1}_t, \pi^{2}_t\, \ldots, \pi^{K}_t\}\) maximizing the total utility over the space of joint actions as:
\begin{align}
    \{\pi^{1}_t, \pi^{2}_t\, \ldots, \pi^{K}_t\}  = & \argmax_{a_{t}^1,a_{t}^2,\dots, a_{t}^K} U(s_t, a_{t}^1,a_{t}^2,\dots, a_{t}^K) \nonumber \\ 
    &\text{ subject to } 0\leq a_{t}^k \leq 1, \quad \sum_{k=1}^K a_{t}^kc_k\leq s_t, \quad \sum_{k=1}^K (1-a_{t}^k)c_k\leq 1-s_t.
\end{align}

The main difference between the MFG and CMFG policies lies in the way that institutions work together. Intuitively, we can view the set of policies as the joint imposition of group-dependent thresholds corresponding to all institutions simultaneously in order to maximize the total utility, as opposed to the earlier decentralized version, where a higher ranked institution imposes its thresholds without consideration of downstream effects on the utilities of lower-ranked institutions. Thus, in CMFG policy, institutions collaborate with each other to maximize their total utility. This allows a central coordinator to distribute the cost of fairness among all institutions, potentially leading to a selection policy that is different from sequential decentralized selection. In particular, the effects of the negative feedback loop observed with the sequential selection under role model reinforcement can be mitigated, as shown in the experimental results in Section~\ref{sec:exps_synt}.

In summary, the discussion of the role model reinforcement variant model serves the following objectives within the context of this study. Firstly, it allows us to simulate a realistic and plausible scenario wherein only a subset of admitted applicants possesses the capacity to exert influence on future participation dynamics. To illustrate this concept, we designate top applicants from each institution as role models, albeit with an acknowledgment of the inherent complexity of truly characterizing successful applicants in real-world situations. This abstraction enables us to underscore the notion that a simplistic strategy of biasing in favor of minority applicants alone does not suffice as a comprehensive solution for fostering future participation. Instead, it underscores the critical importance of institutional support for the admitted applicants to facilitate their growth and eventual success. The effectiveness of this support hinges upon the mechanisms and extent of backing provided by each institution. We show that as a growing number of admitted applicants receive the necessary support to attain a level of success that qualifies them as symbolic role models, the potential for positive reinforcement substantially increases. The precise means by which institutions implement and coordinate these support mechanisms constitutes an open question of paramount significance, particularly for policymakers tasked with shaping equitable participation.
\section{Experimental Evaluation}
We empirically evaluate the proposed MFG policy under different models for the evolution of the applicant pool. We begin with experiments on synthetic data, where the scores are sampled from group-specific Gaussian distributions, followed by learning the score distributions from real-world datasets. 
The MFG and centralized MFG policies are computed numerically and the optimal policies are obtained using the grid search over the space of policies. The number of applicants is kept finite in the experiments.
The code for our experiments is available in a  repository.~\footnote{\href{https://github.com/guldoganozgur/long_term_fairness_pos_reinf}{https://github.com/guldoganozgur/long\_term\_fairness\_pos\_reinf}}

\subsection{Multi-agent framework evaluated on synthetic data}
\label{sec:exps_synt}
\paragraph*{Identical score distributions}
We first consider the case when the score distributions across the two groups in the population are the same. The parameter setting employed for these experiments is listed next. The long-term fairness target is set at $\alpha=0.4$. The number of institutions is $K=3$ with capacities $c_1 = 0.1$, $c_2 = 0.05$ and $c_3 =0.2$, resulting in a total capacity of $0.35$. The score distributions are Gaussian, with means $\mu_0=\mu_1=5$ and variances $\sigma_0^2=\sigma_1^2=1$. The initial mean parameter is $\theta_0 = 0.25$, giving the minority group lower representation in the resource pool. 
The range of the mean parameter is $[0.01,0.99]$, and the mean parameter is projected to the range at each iteration. The state $s_t$ is defined as $\frac{N_t^0}{N_t^0+N_t^1}$, where $N^0_t$ follows a Poisson distribution with parameter $\theta_t N$, $N^1_t$ follows a Poisson distribution with parameter $(1-\theta_t) N$, and $N = 400$. For computational efficiency, once the state is computed, the total number of applicants is fixed to $400$, and the number of applicants from each group is adjusted according to the state and then rounded to the nearest integer for numerical convenience. 
Other parameters include $\lambda=0.75$ and the step-size is fixed as $\eta=0.5$. All the plots are averaged over $200$ instances of the problem.
\begin{figure}[t]
     \centering
     \begin{subfigure}[t]{0.32\columnwidth}
         \centering
         \includegraphics[width=\textwidth]{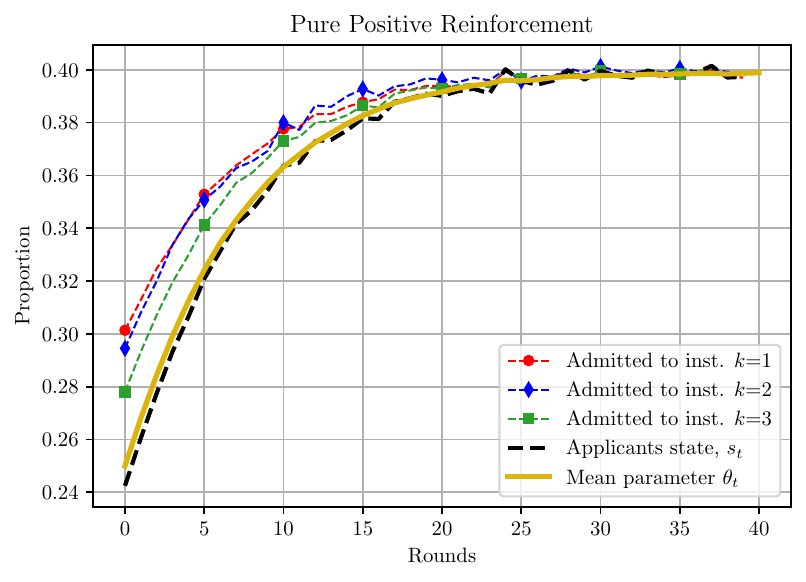}
         \caption{}
         \label{fig:1a}
     \end{subfigure}
     \hfill
     \begin{subfigure}[t]{0.32\columnwidth}
         \centering
         \includegraphics[width=\textwidth]{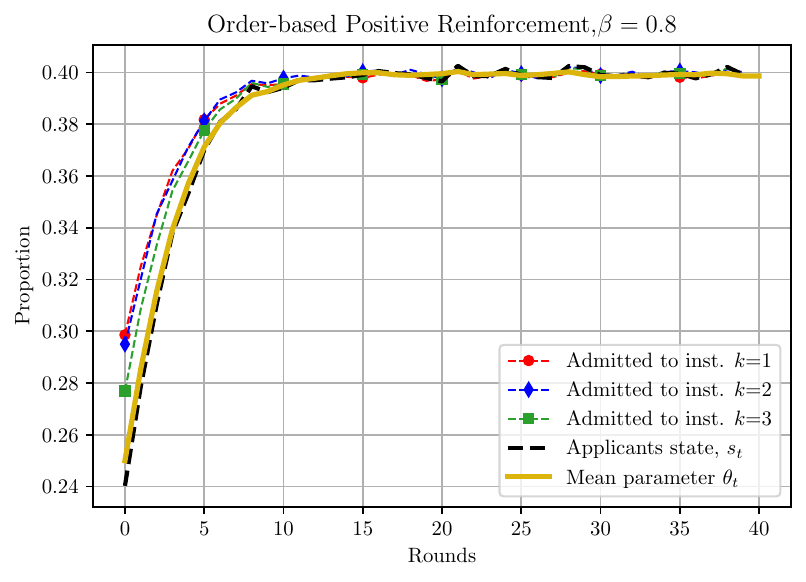}
         \caption{}
         \label{fig:1b}
     \end{subfigure}
     \hfill
     \begin{subfigure}[t]{0.32\columnwidth}
         \centering
         \includegraphics[width=\textwidth]{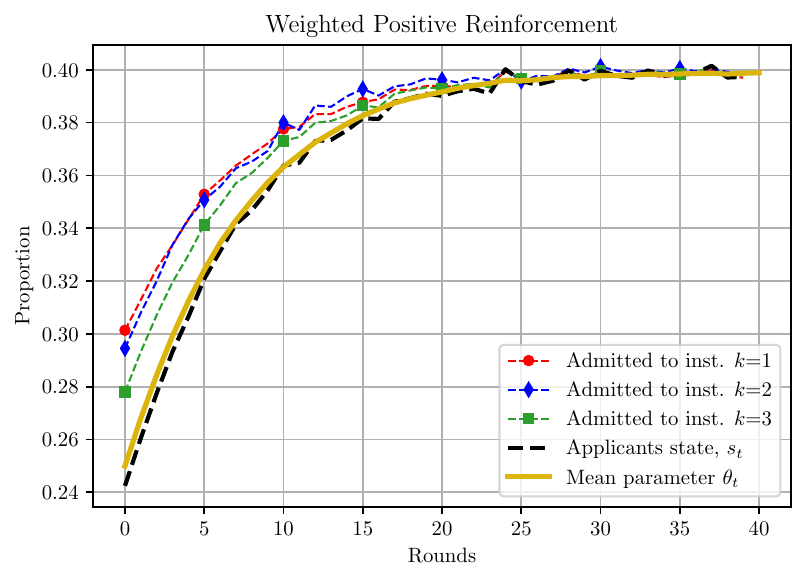}
         \caption{}
         \label{fig:1c}
     \end{subfigure}
     \caption{~\subref{fig:1a} With pure positive reinforcement, the MFG policy reaches long-term fairness when score distributions are identical for both groups.~\subref{fig:1b} The MFG policy attains long-term fairness more quickly with order-based positive reinforcement when $\beta=0.8$.~\subref{fig:1c} Under weighted positive reinforcement, the MFG policy converges when institution weights are equal, for this setting}
    \label{fig:synt_1}
\end{figure}

We demonstrate the evolution of the applicant and admission proportions, and the mean parameter, for the MFG policy under different reinforcement models. In Figure~\ref{fig:1a}, we can observe that under the pure positive reinforcement model, for each institution, the admission proportion is larger than the applicant proportion (state), and hence the weighted MFG policy being larger than the applicant proportion results in the positive reinforcement of the applicant pool (observed through the evolution of mean parameter $\theta_t$) and long-term fairness in admissions as well. Next, we focus on the robustness of the MFG policy under different evolution models. The order-based reinforcement in Figure~\ref{fig:1b} uses $\beta=0.8$, and shows that faster convergence can be achieved in comparison to the pure positive reinforcement model. The weighted positive reinforcement model in Figure~\ref{fig:1c} uses identical weights $w_k=1$ for all institutions, showcasing the pool evolution when equal importance is assigned to every institution. In both these cases, for the setting considered, the MFG policy leads to the achievement of long-term fairness.

Further, we show the evolution of the mean applicant pool parameter under pure positive reinforcement model and the MFG policy for different weights, $\lambda$, allotted to fairness loss in Figure~\ref{fig:pure_pos_identical_lambda}, where we observe that the mean parameter achieves the same long-term fairness target in equilibrium, independent of the hyperparameter $\lambda$, under identical score distributions, and long as $\lambda>0$. However, the rate of convergence depends on $\lambda$.

\begin{figure}[t]{} 
  \centering
  \includegraphics[width=0.32\textwidth]{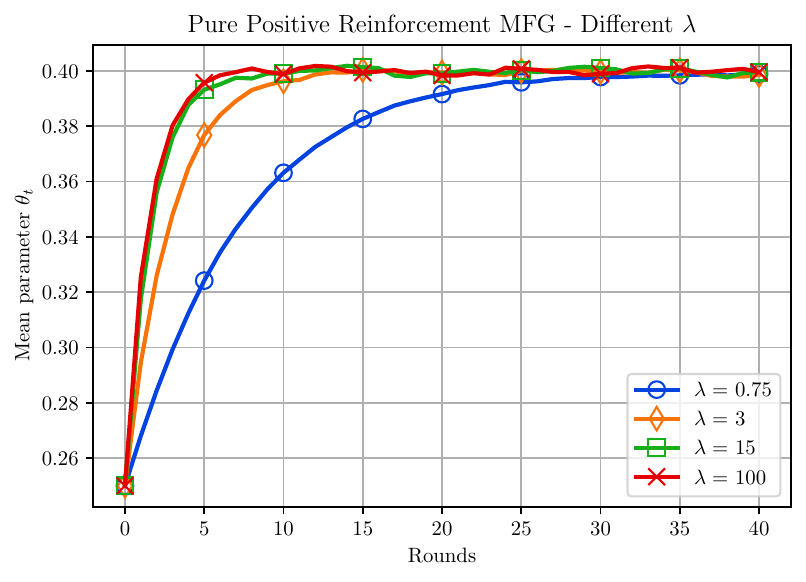} 
  \caption{MFG policy under identical scores reaches long-term fairness target, independent of $\lambda$.}
  \label{fig:pure_pos_identical_lambda} 
\end{figure}

\begin{figure}[t]
    \centering
     \begin{subfigure}[t]{0.32\columnwidth}
         \centering
         \includegraphics[width=\textwidth]{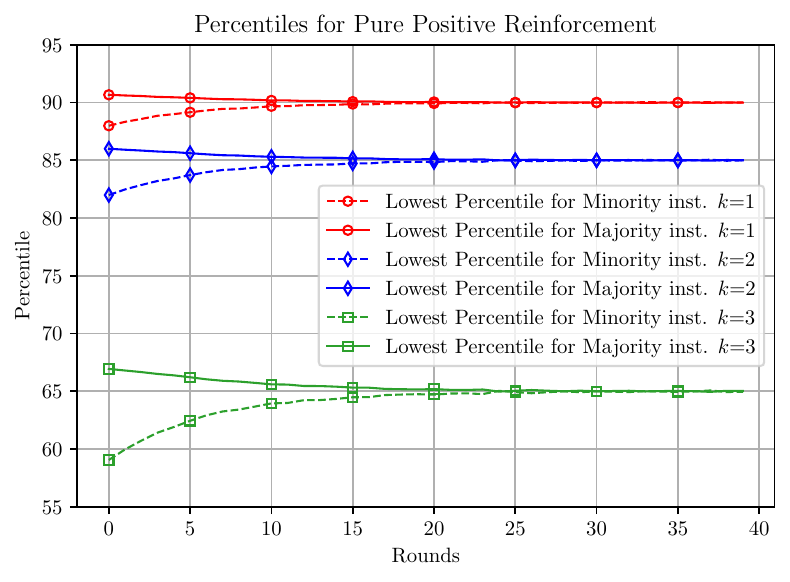}
         \caption{}
         \label{fig:pure_pos_percentiles}
     \end{subfigure}
     \begin{subfigure}[t]{0.32\columnwidth}
         \centering
         \includegraphics[width=\textwidth]{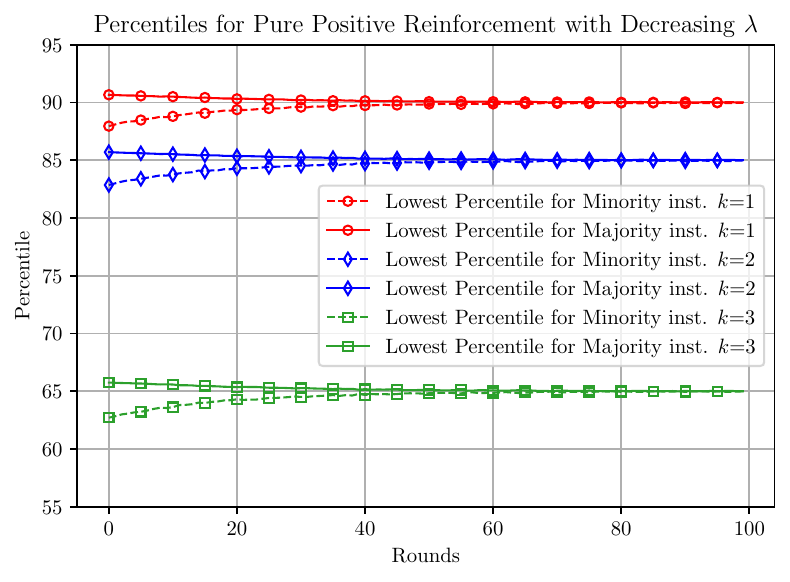}
         \caption{}
         \label{fig:pure_pos_percentiles_dec_lambda}
     \end{subfigure}
    \begin{subfigure}[t]{0.32\columnwidth}
         \centering
         \includegraphics[width=\textwidth]{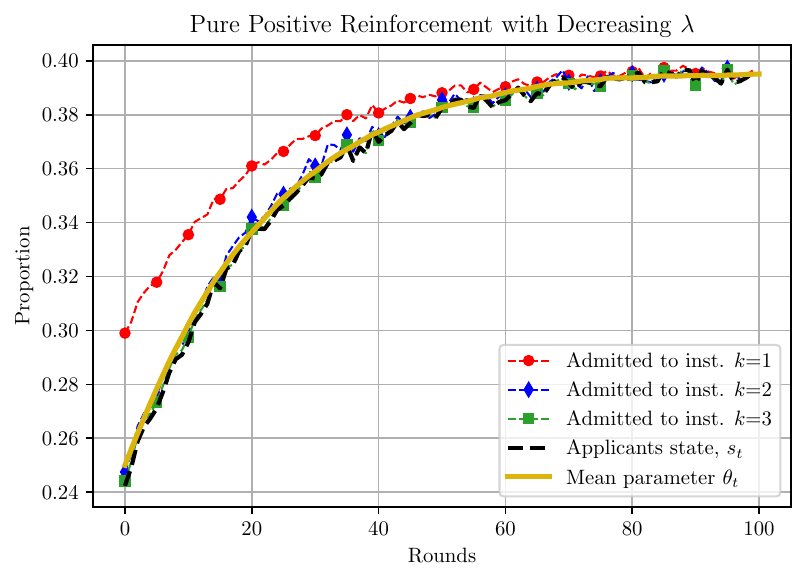}
         \caption{}
         \label{fig:pure_pos_dec_lambda}
     \end{subfigure}
    \caption{The score percentile of the admitted applicant with the least score, from each group and for each institution, under the setting where (i) $\lambda=0.75$ for all institutions (in~\subref{fig:pure_pos_percentiles}) and (ii) $\lambda$ is in decreasing order for the three institutions (in~\subref{fig:pure_pos_percentiles_dec_lambda}). The evolution of mean parameter and admission proportions under the case of decreasing $\lambda$ is in~\subref{fig:pure_pos_dec_lambda}}.
    \label{}
\end{figure}

Next, we gain valuable insights into the operational dynamics of the MFG policy. Our objective is to comprehend the trade-offs made, in terms of scores of admitted applicants of the majority and minority groups, to foster fairness. Figure~\ref{fig:pure_pos_percentiles} depicts the percentile at which the admitted applicant with the lowest score is positioned, for both the groups and all institutions, when $\lambda=0.75$ across all participating institutions. Next, in Figure~\ref{fig:pure_pos_percentiles_dec_lambda}, we extend our examination to a scenario in which the $\lambda$ values are decremented across institutions as $[0.75,0.375,0.1875]$, signifying that lower-ranked institutions temporarily de-prioritize diversity to minimize a significant decline in the quality of admitted applicants. Consequently, this approach illustrates how a judicious delay in the pursuit of immediate fairness objectives by lower-ranked institutions can assist in averting a large drop in their admission standards. However, this strategy of decreasing $\lambda$ impacts the overall convergence rate (see Fig.~\ref{fig:pure_pos_dec_lambda}), as all institutions experience a deferment in the applicant pool reaching the long-term fairness objective. This serves as an example of the delicate balance that institutions must navigate between the convergence to fairness targets and maximization of short-term rewards. Nonetheless, the precise strategies and mechanisms to be adopted by individual institutions in achieving this balance remains open.

\begin{figure}[t!]
    \centering

  \begin{subfigure}[t]{0.24\columnwidth}
         \centering
         \includegraphics[width=\textwidth]{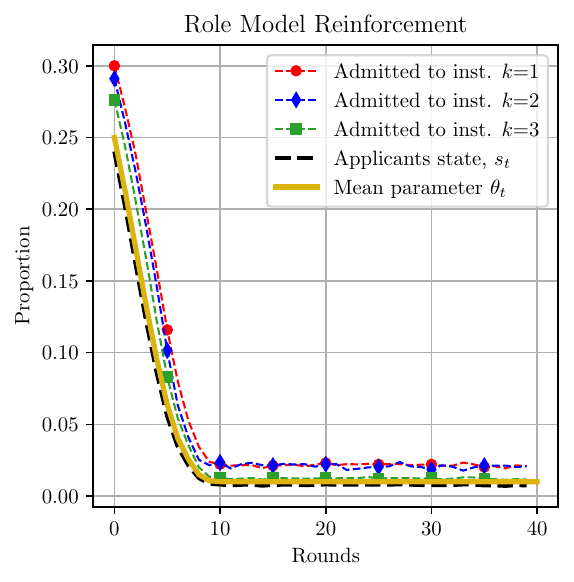}
         \caption{}
         \label{fig:mfg_role_model_identical}
     \end{subfigure}
     \begin{subfigure}[t]{0.24\columnwidth}
         \centering
         \includegraphics[width=\textwidth]{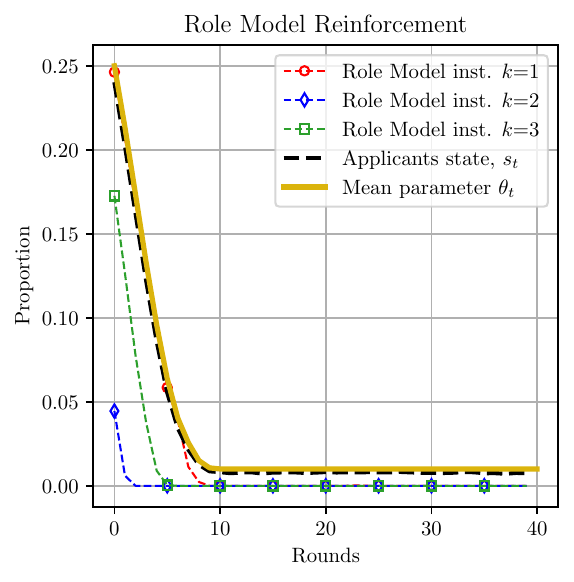}
         \caption{}
         \label{fig:mfg_role_model_identical_role_models}
     \end{subfigure}
     \begin{subfigure}[t]{0.24\columnwidth}
         \centering
         \includegraphics[width=\textwidth]{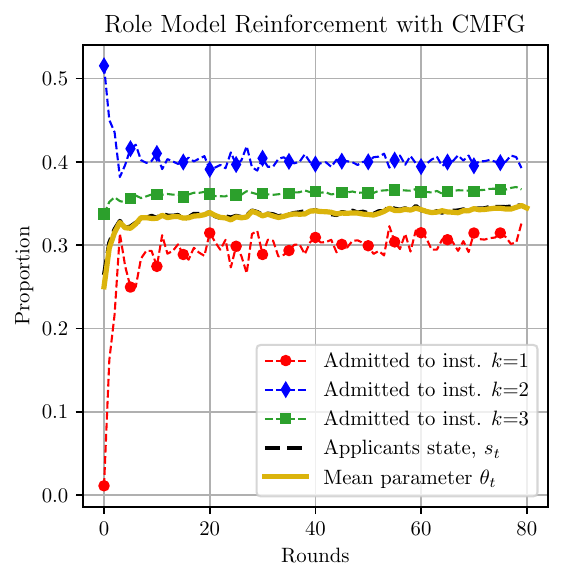}
         \caption{}
         \label{fig:cmfg_role_model_identical}
     \end{subfigure}
     \begin{subfigure}[t]{0.24\columnwidth}
         \centering
         \includegraphics[width=\textwidth]{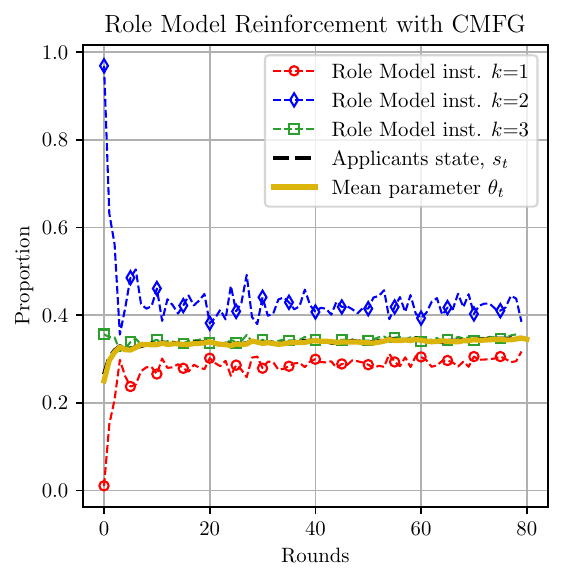}
         \caption{}
         \label{fig:cmfg_role_model_identical_roles}
     \end{subfigure}
     \caption{~\subref{fig:mfg_role_model_identical} MFG policy creates a negative feedback loop under the role model reinforcement.~\subref{fig:mfg_role_model_identical_role_models} The evolution of the proportions of role models for each institution, under MFG policy.~\subref{fig:cmfg_role_model_identical} CMFG policy could potentially alleviate negative feedback under role model reinforcement.~\subref{fig:cmfg_role_model_identical_roles}The evolution of the proportions of role models for each institution, under CMFG policy.}
\end{figure}

\noindent
\paragraph*{Potential for negative feedback} The evolution dynamics of role model reinforcement, under the MFG policy is examined here. We consider a scenario where role model reinforcement is applied with a parameter of $r=0.5$, i.e., the admissions with scores above the median score in the respective institution are considered role models for the resource pool evolution. As seen in Figure~\ref{fig:mfg_role_model_identical}, the MFG policy creates a negative feedback loop, causing a significant decrease of the minority group in the applicant pool over time. This is because the role model proportions $r_t^k$ are such that the weighted role model policy is consistently smaller than the state. In particular, as seen in Figure~\ref{fig:mfg_role_model_identical_role_models}, the second and the third institutions see a very small fraction of role models among the minority group, since the first institution admits the top minority applicants. These effects could potentially be alleviated by considering a centralized policy, such as the CMFG policy, whose evolution for the admitted applicants and the proportion of the role models under this setting are shown in Figure~\ref{fig:cmfg_role_model_identical} and~\ref{fig:cmfg_role_model_identical_roles}. Although the initial dynamics between the institutions under CMFG policy are not desirable in the real-world and not well understood, it shows a potential of ultimately attaining an equilibrium close to the long-term fairness target. We also remark that the adverse effects under role model reinforcement in conjunction with the MFG policy could be avoided by increasing the $r$ parameter, which essentially means that the institutions must design intervention mechanisms or remedies to support a large fraction of the admissions to eventually be successful in society, and influence others by being true role models. 

For completeness, we show the evolution of applicant and admission proportions under the CMFG policy under pure positive reinforcement, order-based and weighted positive reinforcement models in Appendix~\ref{app:iden_score_dists}.

\paragraph*{Distinct score distributions}
\begin{figure}[t]
     \centering
     \begin{subfigure}[t]{0.32\columnwidth}
         
         \includegraphics[width=\textwidth]{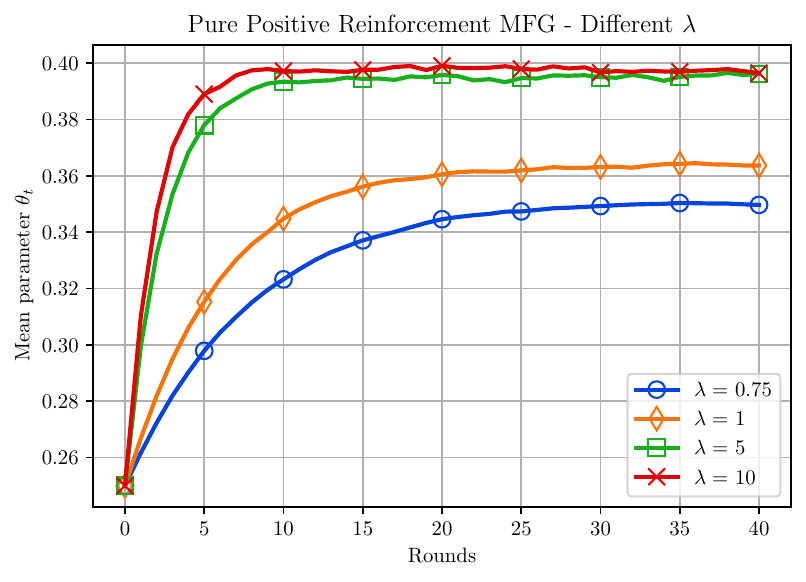}
         \caption{}
         \label{fig:5a}
     \end{subfigure}
     \hfill
     \begin{subfigure}[t]{0.32\columnwidth}
         \centering
         \includegraphics[width=\textwidth]{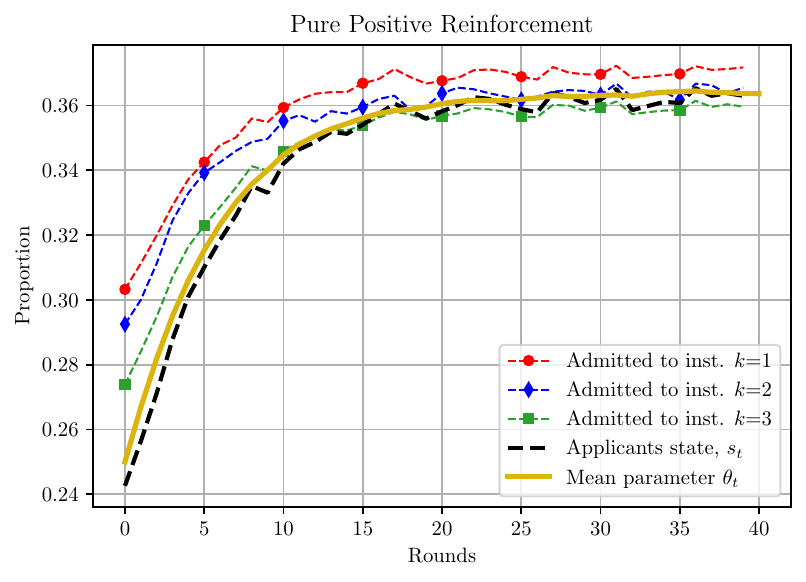}
         \caption{}
         \label{fig:5b}
     \end{subfigure}
     \hfill
     \begin{subfigure}[t]{0.32\columnwidth}
         \centering
         \includegraphics[width=\textwidth]{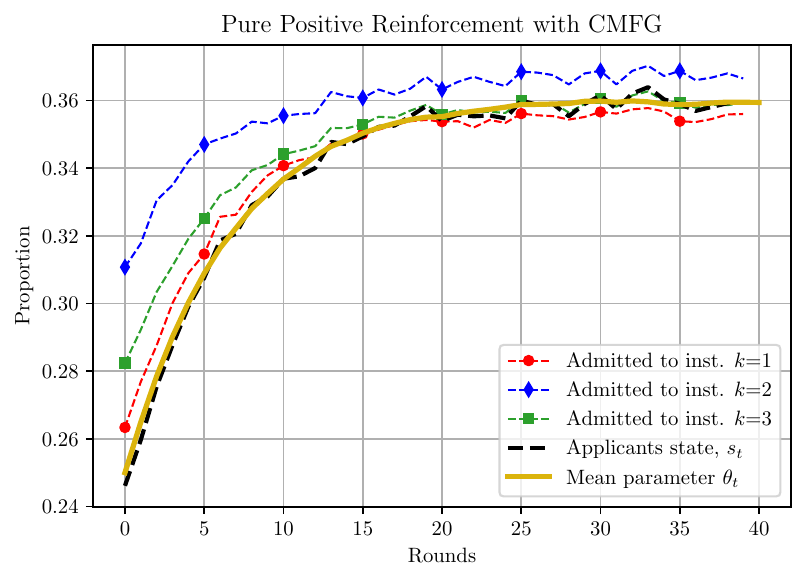}
         \caption{}
         \label{fig:5c}
     \end{subfigure}
    \caption{~\subref{fig:5a} Convergence of the mean parameter under the MFG policy with pure positive reinforcement is impacted by the fairness loss coefficient, $\lambda$, when score distributions are distinct.~\subref{fig:5b} MFG policy reaches an equilibrium with pure positive reinforcement when score distributions are different.~\subref{fig:5c} CMFG also reaches an equilibrium, albeit with different decisions.}
    \label{fig:synt_3}
\end{figure} 
We now examine the scenario where the score distributions for the groups are distinct. We assume that the minority group has a slightly lower mean and higher variance than the majority group ($\mu_0=4.9,\mu_1=5$ and $\sigma_0^2=1.1,\sigma_1^2=1$). 
The other parameters remain the same as in the first scenario, with a long-term fairness target of $\alpha=0.4$, capacities $c_1 = 0.1$, $c_2 = 0.05$ and $c_3 = 0.2$, an initial mean parameter $\theta_0=0.25$, and fixed step size $\eta=0.5$.

Figure~\ref{fig:5a} shows the evolution of the mean parameter under different weights, $\lambda$, for fairness loss, under non-identical distributions. We observe that by altering $\lambda$, the equilibrium point can be varied. Thus, a higher value of $\lambda$ is required to achieve the long-term fairness goal. We then show in Figures~\ref{fig:5b} and~\ref{fig:5c}, the evolution of the applicant pool, admission proportions, and their convergence under the pure positive reinforcement model, under MFG and CMFG policies respectively, with $\lambda = 1$. They both converge to similar equilibrium although the admission decisions in initial rounds are quite distinct. We observe a behavior of negative feedback under the role-model reinforcement model under distinct score distributions as well. The results showing negative feedback and the alleviation of this effect through the CMFG policy are deferred to Appendix~\ref{sec:app_dist_dists}. In the next section, we will examine cases where the score distributions are significantly different, obtained from a real-world dataset.

\subsection{Multi-agent framework evaluated on semi-synthetic dataset}
\begin{figure}[t]
     \centering
     \begin{subfigure}[t]{0.3\columnwidth}
         \centering
         \includegraphics[width=\textwidth]{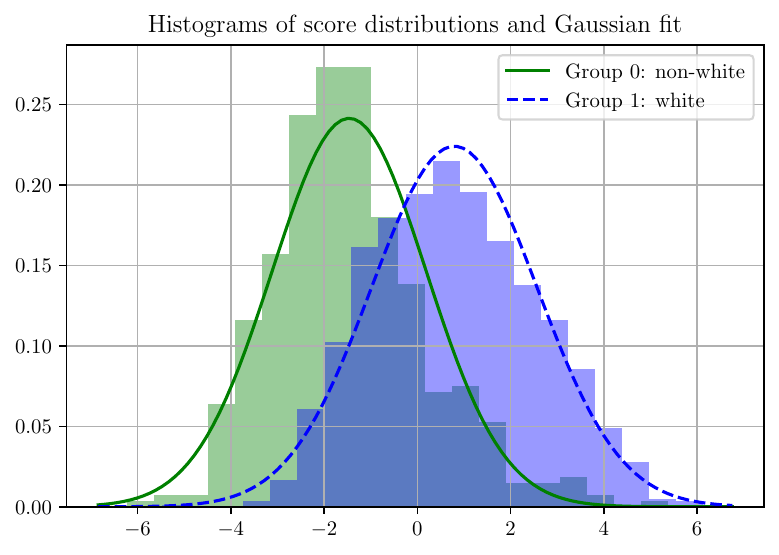}
         \caption{}
         \label{fig:lsa_hists}
     \end{subfigure}
     \hfill
     \begin{subfigure}[t]{0.3\columnwidth}
         \centering
         \includegraphics[width=\textwidth]{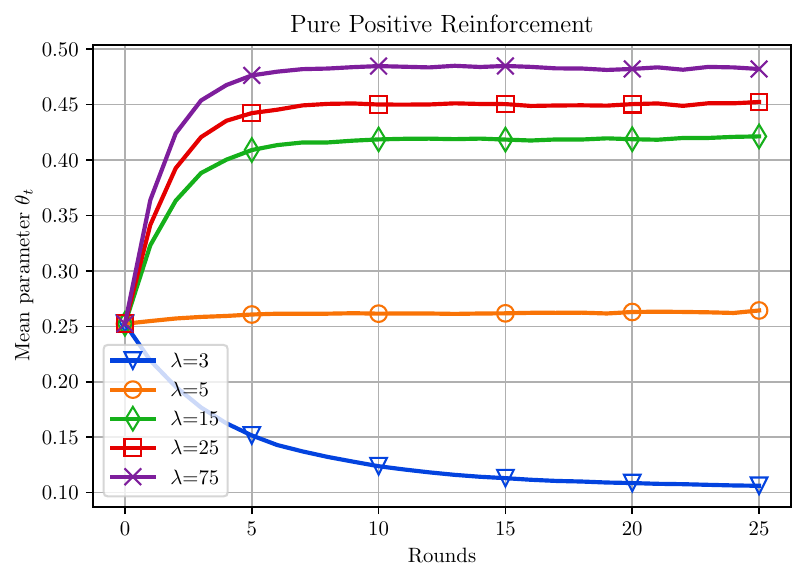}
         \caption{}
         \label{fig:lsa_pure_pos_mfg}
     \end{subfigure}
     \hfill
     \begin{subfigure}[t]{0.3\columnwidth}
         \centering
         \includegraphics[width=\textwidth]{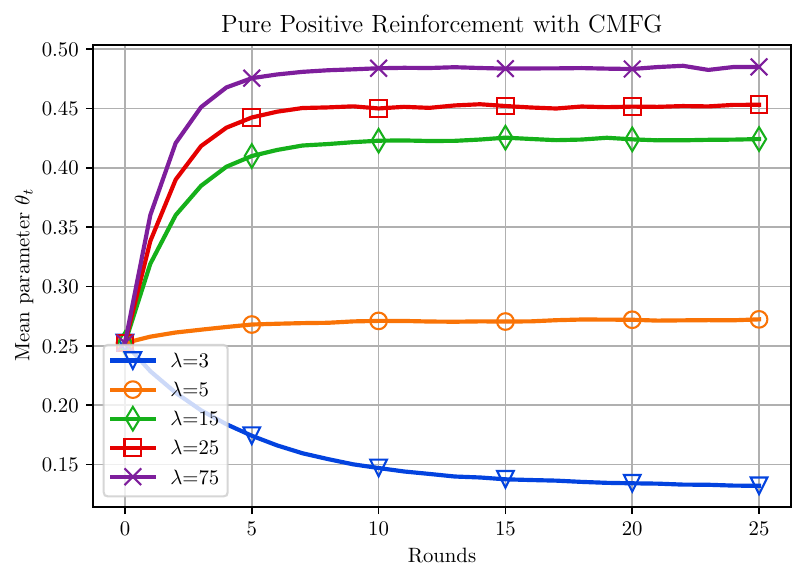}
         \caption{}
         \label{fig:lsa_pure_pos_cmfg}
     \end{subfigure}
    \caption{~\subref{fig:lsa_hists} depicts the score distributions of the white and non-white groups in the law school dataset, along with their Gaussian approximations.~\subref{fig:lsa_pure_pos_mfg},~\subref{fig:lsa_pure_pos_cmfg} show the evolution of the mean parameter, $\theta_t$, and its equilibrium with varying values of the fairness loss coefficient, $\lambda$, under the MFG and centralized MFG policies respectively under pure positive reinforcement.}
    \label{fig:law_1}
\end{figure}

In this section, we report on experiments on the law school bar study dataset~\cite{lsac_bar}. The dataset is used to construct the initial setup of the experiments as defined in the previous section, therefore we refer it to as a semi-synthetic dataset. This dataset contains information collected by the Law School Admission Council from law schools in the US, including information on whether an applicant passed the bar exam based on features such as LSAT scores, undergraduate GPA, law school GPA, race, gender, family income, age, and others. We take race as the protected attribute and simplify it to a binary classification problem by grouping all races except "white" into the "non-white" category, serving as the minority group (group $0$) with only $25\%$ representation in the dataset. The dataset used in our experiments contains around 1800 instances and we followed the same pre-processing steps at~\cite{kearns2018preventing}. Next, we follow the procedure outlined in~\cite{puranik2022dynamic} to obtain the score distributions for each group. After pre-processing the data, we fit a logistic regression model to approximate the score distributions as Gaussians. This gives us the means $\mu_0=-1.46$ and $\mu_1=0.79$, and the variances $\sigma_0^2=2.73$ and $\sigma_1^2=3.16$, as seen in Figure~\ref{fig:lsa_hists}. The long-term fairness target is set to $0.5$, with three institutions having capacities of $c_1 = 0.15$, $c_2 =0.10$ and $c_3 =0.05$. The step size, $\eta$, is fixed at 0.5. Our first examination is of the pure positive reinforcement model. Figures~\ref{fig:lsa_pure_pos_mfg} and~\ref{fig:lsa_pure_pos_cmfg} depict the mean parameter, $\theta_t$ at equilibrium for the MFG and CMFG policies, respectively, with varying fairness loss coefficients, $\lambda$. A higher value of $\lambda$ is needed to achieve the long-term fairness goal for both policies, with the CMFG showing a slightly better equilibrium point at a smaller value of $\lambda$ (for instance, $\lambda = 3$). We defer the evaluation under role-model reinforcement model to Appendix~\ref{app:role_model_lsa_dataset}. 

Our experiments with distinct distributions illustrate that the principles of reinforcement and the policies developed in this paper hold even when the score distributions are non-identical.

\section{Conclusion}
This paper studies the evolution of long-term fairness in a selection setting with multiple decision-makers choosing from a common pool. We have shown that the Multi-agent Fair-Greedy (MFG) policy does succeed in achieving long-term fairness targets under the model of pure positive reinforcement. However, when we set a higher bar for successful influence via the role-model reinforcement model, the minority group may actually experience negative feedback under MFG policy, and ultimately exit the selection process. Centralized coordination among the institutions could potentially alleviate this problem, raising the question of whether we can design mechanisms for competing institutions to collaborate (without laying themselves open to charges of collusion) in order to advance long-term fairness in society at large.  

We hope that this work, despite the simplicity of the models considered, stimulates continuing discussion on the long-term societal impact of automated decision-making in a multi-agent setting, and how we can shape it. The sensitivity of our simple system to the model for evolution motivates a concerted effort to launch real-world experiments and data collection in which algorithm designers collaborate closely with social scientists and policymakers. An important complementary effort is to pursue analytical insights for more complex models that capture different aspects of the real world. For example, while our current concept of role model is based on the relative score upon admission, qualifications upon admission are often not a predictor of ultimate success; support mechanisms provided by the institution may be more important. Can we derive insights from a plausible model for such support mechanisms? Similarly, is there a qualitative difference in our conclusions if we relax the simplifying assumption of strict institutional rankings determining preferences for all applicants?

\bibliographystyle{apalike}
\bibliography{bib_data}

\newpage
\appendix
\section*{Appendices}
\section{Optimizing the score-based reward under MFG policy}
\label{sec:app_opt_greedy_action}
We provide below a detailed proof of Theorem~\ref{thm:opt_score_act} for the optimal action that maximizes the score-based reward of each institution.

\begin{proof}
Assuming the score distributions' CDF is denoted by $\mathcal{F}$. By the Lemma~\ref{lem:concave_score_reward}, the score-based reward function of the institution $k$ is concave and it is maximized when the absolute difference between the lower thresholds of both groups is minimized, $|t_{0}^{\text{k,S,low}} - t_1^{\text{k,S,low}}|$. This is equivalent to the following minimization;
\begin{equation}
    \min_{a_t^k\in \mathcal{A}^k_t} \left|\mathcal{F}^{-1}\left(1-\frac{\bar{c}_t^{0,k}+a_tc_k}{s_t}\right) - \mathcal{F}^{-1}\left(1-\frac{\bar{c}_t^{1,k}+(1-a_t)c_k}{1-s_t}\right)\right|, \forall k \in \{2,\ldots, K\}
\end{equation}
where $\bar{c}_t^{0,k} = \sum_{j=1}^{k-1}\pi_t^jc_j$ and $\bar{c}_t^{1,k} = \sum_{j=1}^{k-1}(1-\pi_t^j)c_j$. Due to the monotonicity of the inverse CDF, the optimal action that maximizes the score-based reward is equivalent to minimizing the following:
\begin{equation}
    \min_{a_t^k\in \mathcal{A}^k_t} \left|\left(1-\frac{\bar{c}_t^{0,k}+a_tc_k}{s_t}\right) - \left(1-\frac{\bar{c}_t^{1,k}+(1-a_t)c_k}{1-s_t}\right)\right|, \forall k \in \{2,\ldots, K\}
\end{equation}
This can be simplified to the following:
\begin{equation}
    \min_{a_t^k\in \mathcal{A}^k_t} \left|\frac{\bar{c}_t^{0,k}+a_tc_k}{s_t} - \frac{\bar{c}_t^{1,k}+(1-a_t)c_k}{1-s_t}\right|, \forall k \in \{2,\ldots, K\}
\end{equation}
The argument of the minimum can be multiplied by $s_t(1-s_t)$ to simplify the expression as follows:
\begin{equation}
    \min_{a_t^k\in \mathcal{A}^k_t} \left|a_tc_k(1-s_t)-(1-a_t)c_ks_t - \sum_{j=1}^{k-1}(1-\pi_t^j)c_js_t+\sum_{j=1}^{k-1}\pi_t^jc_j(1-s_t) \right|, \forall k \in \{2,\ldots, K\}
\end{equation}
The terms with the action $a_t$ and the terms without the action $a_t$ can be grouped as follows:
\begin{equation}
    \min_{a_t^k\in \mathcal{A}^k_t} \left|a_t - s_t - \frac{1}{c_k}\left(\sum_{j = 1}^{k-1}c_j (s_t - \pi^{j}_t )\right)\right| = \left[ s_t + \frac{1}{c_k}\left( \sum_{j = 1}^{k-1}c_j (s_t - \pi^{j}_t )\right)\right]_{\mathcal{A}_{t}^k}, \forall k \in \{2,\ldots, K\}
\end{equation}
Therefore, the optimal action that maximizes the score-based reward is as stated in the theorem. It can also be seen from the Lemma~\ref{lem:concave_score_reward} that the score-based reward function of the institution $1$ is maximized when the lower thresholds of both groups are equal, $t_{0}^{\text{1,S,low}} = t_1^{\text{1,S,low}}.$ This is satisfied when $a_t^1 = s_t$.
\end{proof}
\subsection{Technical Lemmas}
In this section, we show that the score-based reward function of each institution is concave and has a unique action that maximizes its reward.
\begin{lemma}\label{lem:concave_score_reward}
    Under Assumption~\ref{asm:large}, the score-based reward function, $R_{k}(s_t, a_{t})$, is concave and it is maximized when the absolute difference between the lower thresholds of both groups is minimized, $|t_{0}^{\text{k,S,low}} - t_1^{\text{k,S,low}}|$. This is equivalent to the following minimizing;
    \begin{equation}
        \min_{a_t^k\in \mathcal{A}^k_t} \left|\mathcal{F}_0^{-1}\left(1-\frac{\bar{c}_t^{0,k}+a_tc_k}{s_t}\right) - \mathcal{F}_1^{-1}\left(1-\frac{\bar{c}_t^{1,k}+(1-a_t)c_k}{1-s_t}\right)\right|, \forall k \in \{2,\ldots, K\}
    \end{equation}
    where $\bar{c}_t^{0,k} = \sum_{j=1}^{k-1}\pi_t^jc_j$ and $\bar{c}_t^{1,k} = \sum_{j=1}^{k-1}(1-\pi_t^j)c_j$ and for $k=1$, the action that maximizes the score-based reward is equivalent to minimizing $\left|\mathcal{F}_0^{-1}\left(1-\frac{a_tc_1}{s_t}\right) - \mathcal{F}_1^{-1}\left(1-\frac{(1-a_t)c_1}{1-s_t}\right)\right|$.
\end{lemma}
\begin{proof}
Dropping the superscript $k$ in the notation of the action for brevity, the score-based reward of the $k^{th}$ institution is:
\begin{equation}
    R_{k}(s_t, a_{t}) = a_t \frac{\mathbb{E}\left[ \sum_{i = {m}^{0}_{k,t}+1}^{{m}^{0}_{k,t}+ a_{t} A_{k,t}} X^0_{(i)}\right]}{a_t A_{k,t}} + (1-a_t) \frac{\mathbb{E} \left[\sum_{i = {m}^{1}_{k,t}+ 1}^{{m}^{1}_{k,t}+(1 - a_{t})A_{k,t}} X^1_{(i)} \right]}{(1-a_t)  A_{k,t}}
\end{equation}
The number of admitted applicants is approximated as $A_{k,t}^0 = \lfloor a_t A_{k,t}\rfloor \approx a_t A_{k,t}$, as we are considering that the number of applicants is large. Under this regime, the collection of scores is in accordance with their respective group-specific score distributions $\mathcal{P}^g$. Following the idea in~\cite{puranik2022dynamic}, the average scores of selected applicants from each group can be represented by the following conditional expectations:
\begin{equation}
    \lim_{N \rightarrow \infty} \frac{\sum_{i = {m}^{0}_{k,t}+1}^{{m}^{0}_{k,t}+ a_{t} A_{k,t}} X^0_{(i)}}{a_t A_{k,t}} = \mathbb{E}\left[X^0 \mid t_0^{\text{k,S,low}} \leq X^0 \leq t_0^{\text{k,S,up}}\right]
\end{equation}
\begin{equation}
    \lim_{N \rightarrow \infty} \frac{\sum_{i = {m}^{1}_{k,t}+1}^{{m}^{1}_{k,t}+ (1-a_{t}) A_{k,t}} X^1_{(i)}}{(1-a_t) A_{k,t}} = \mathbb{E}\left[X^1 \mid t_1^{\text{k,S,low}} \leq X^1 \leq t_1^{\text{k,S,up}}\right]
\end{equation}
Then, the score-based reward function of the institution $k$ can be written as follows:
\begin{equation}\label{eq:greedy_reward_multi}
    R_g^k(s_t,a_t) = a_t\mathbb{E}\left[X^0 \mid t_0^{\text{k,S,low}} \leq X^0 \leq t_0^{\text{k,S,up}}\right]+(1-a_t)\mathbb{E}\left[X^1 \mid t_1^{\text{k,S,low}} \leq X^1 \leq t_1^{\text{k,S,up}}\right]
\end{equation}
We denote the cumulative distribution function (CDF) of the score-distributions by $\mathcal{F}_0$ and $\mathcal{F}_1$. The tail of the distributions beyond the upper thresholds represents the proportion of the applicants from the particular group who have been already admitted by the better-ranked institutions.
\begin{align}
    1 - \mathcal{F}_0(t_0^{\text{k,S,up}}) &= \frac{\sum_{j = 1}^{k-1} c_j  \pi^{j}_{t}}{s_t} \implies t_0^{\text{k,S,up}} = \mathcal{F}_0^{-1}\left(1-\frac{\bar{c}_t^{0,k}}{s_t}\right)\label{eq:threshold_upper_1_multi}\\
    1 - \mathcal{F}_1(t_1^{\text{k,S,up}}) &= \frac{\sum_{j = 1}^{k-1} c_j  (1-\pi^{j}_{t})}{(1-s_t) }\implies t_1^{\text{k,S,up}} = \mathcal{F}_1^{-1}\left(1-\frac{\bar{c}_t^{1,k}}{1-s_t}\right)\label{eq:threshold_upper_2_multi}
\end{align}
Moreover, the area between the lower and upper thresholds in the distributions signifies the proportion of applicants admitted by the specific group from the total number of applicants belonging to that group at the $k^{th}$ institution. Thus, we have:
\begin{align}\label{eq:threshold_diff_multi}
    \mathcal{F}_0(t_0^{\text{k,S,up}}) - \mathcal{F}_0(t_0^{\text{k,S,low}}) = \frac{a_tc_k}{s_t}  \quad \text{and} \quad
    \mathcal{F}_1(t_1^{\text{k,S,up}}) - \mathcal{F}_1(t_1^{\text{k,S,low}}) = \frac{(1-a_t)c_k}{1-s_t} 
\end{align}
Then using the upper thresholds, we can write the lower thresholds as follows:
\begin{align}
    1 - \mathcal{F}_0(t_0^{\text{k,S,low}}) = \frac{\bar{c}_t^{0,k}}{s_t}+\frac{a_tc_k}{s_t}  \quad \text{and} \quad
    1 - \mathcal{F}_1(t_1^{\text{k,S,low}}) = \frac{\bar{c}_t^{1,k}}{1-s_t}+\frac{(1-a_t)c_k}{1-s_t} 
\end{align}
Then, we can write the lower thresholds as follows:
\begin{align}\label{eq:threshold_def_multi}
    t_0^{\text{k,S,low}} = \mathcal{F}_0^{-1}\left(1-\frac{\bar{c}_t^{0,k}+a_tc_k}{s_t}\right)   \quad \text{and} \quad
    t_1^{\text{k,S,low}} = \mathcal{F}_1^{-1}\left(1-\frac{\bar{c}_t^{1,k}+(1-a_t)c_k}{1-s_t}\right)
\end{align}
We can write the score reward function of the institution $k$ as follows:
\begin{equation}
    R_g^k(s_t,a_t) = a_t\int_{t_0^{\text{k,S,low}}}^{t_0^{\text{k,S,up}}}x_0\frac{f_0(x_0)}{\int_{t_0^{\text{k,S,low}}}^{t_0^{\text{k,S,up}}}f_0(x_0)\mathrm{d}x_0}\mathrm{d}x_0+(1-a_t)\int_{t_1^{\text{k,S,low}}}^{t_1^{\text{k,S,up}}}x_1\frac{f_1(x_1)}{\int_{t_1^{\text{k,S,low}}}^{t_1^{\text{k,S,up}}}f_1(x_1)\mathrm{d}x_1}\mathrm{d}x_1
\end{equation}
We can move the common denominator out of the integral since it is a constant:
\begin{equation}
    R_g^k(s_t,a_t) = \frac{a_t}{\int_{t_0^{\text{k,S,low}}}^{t_0^{\text{k,S,up}}}f_0(x_0)\mathrm{d}x_0}\int_{t_0^{\text{k,S,low}}}^{t_0^{\text{k,S,up}}}x_0f_0(x_0)\mathrm{d}x_0+\frac{(1-a_t)}{\int_{t_1^{\text{k,S,low}}}^{t_1^{\text{k,S,up}}}f_1(x_1)\mathrm{d}x_1}\int_{t_1^{\text{k,S,low}}}^{t_1^{\text{k,S,up}}}x_1f_1(x_1)\mathrm{d}x_1
\end{equation}
Then, the denominator can be written in terms of the CDF as follows:
\begin{align}
    \scalebox{0.95}{$
\begin{aligned}
    R_g^k(s_t,a_t) = \frac{a_t}{\mathcal{F}_0(t_0^{\text{k,S,up}})-\mathcal{F}_0(t_0^{\text{k,S,low}})}\int_{t_0^{\text{k,S,low}}}^{t_0^{\text{k,S,up}}}x_0f_0(x_0)\mathrm{d}x_0+\frac{(1-a_t)}{\mathcal{F}_1(t_1^{\text{k,S,up}})-\mathcal{F}_1(t_1^{\text{k,S,low}})}\int_{t_1^{\text{k,S,low}}}^{t_1^{\text{k,S,up}}}x_1f_1(x_1)\mathrm{d}x_1
\end{aligned}$}
\end{align}
Using the equations~\eqref{eq:threshold_diff_multi}, we can write the score-based reward function of the institution $k$ as follows:
\begin{equation}\label{eq:greedy_reward_multi_2}
    R_g^k(s_t,a_t) = \frac{s_t}{c_k}\int_{t_0^{\text{k,S,low}}}^{t_0^{\text{k,S,up}}}x_0f_0(x_0)\mathrm{d}x_0+\frac{(1-s_t)}{c_k}\int_{t_1^{\text{k,S,low}}}^{t_1^{\text{k,S,up}}}x_1f_1(x_1)\mathrm{d}x_1
\end{equation}
The upper thresholds are not dependent on the current institution's action.
Then, we can use the definitions of lower thresholds from equations~\eqref{eq:threshold_def_multi} to write the score-based reward function of the institution $k$ as follows:
\begin{equation}
    R_g^k(s_t,a_t) = \frac{s_t}{c_k}\int_{\mathcal{F}_0^{-1}\left(1-\frac{\bar{c}_t^{0,k}+a_tc_k}{s_t}\right)}^{t_0^{\text{k,S,up}}}x_0f_0(x_0)\mathrm{d}x_0+\frac{(1-s_t)}{c_k}\int_{\mathcal{F}_1^{-1}\left(1-\frac{\bar{c}_t^{1,k}+(1-a_t)c_k}{1-s_t}\right)}^{t_1^{\text{k,S,up}}}x_1f_1(x_1)\mathrm{d}x_1
\end{equation}
Using the fundamental theorem of calculus, we can write the derivative of the score-based reward function of the institution $k$ as follows:
\begin{equation}\label{eq:greedy_reward_multi_derivative}
    \diff{R_g^k(s_t,a_t)}{{a_t}} = \mathcal{F}_0^{-1}\left(1-\frac{\bar{c}_t^{0,k}+a_tc_k}{s_t}\right) - \mathcal{F}_1^{-1}\left(1-\frac{\bar{c}_t^{1,k}+(1-a_t)c_k}{1-s_t}\right) = t_0^{\text{k,S,low}} - t_1^{\text{k,S,low}}
\end{equation}
Then, we can write the second derivative as follows:
\begin{equation}\label{eq:greedy_reward_multi_second_derivative}
    \frac{\mathrm{d}^2R_g^k(s_t,a_t)}{{\mathrm{d}a_t^{k}}^2} = -\frac{c_k}{s_t}\frac{1}{f_{0}\left(\mathcal{F}_0^{-1}\left(1-\frac{\bar{c}_t^{0,k}+a_tc_k}{s_t}\right)\right)}-\frac{c_k}{(1-s_t)}\frac{1}{f_{1}\left(\mathcal{F}_1^{-1}\left(1-\frac{\bar{c}_t^{1,k}+(1-a_t)c_k}{1-s_t}\right)\right)}
\end{equation}
Since the PDF functions are non-negative, $c_k \in (0,1]$, and $s_t \in (0,1)$, the second derivative is non-positive. This means the score-based reward function of the institution $k$ is concave.
\newline
The score-based reward function of the institution $k$ is concave and its first derivative is monotone and defined on the interval $[\max(0, 1-\frac{1-s_t-\bar{c}_t^{1,k}}{c_k}), \min(1, \frac{s_t-\bar{c}_t^{0,k}}{c_k})]$. Then, the score-based reward function of the institution $k$ is maximized when the absolute difference between the lower thresholds of both groups is minimized, $|t_{0}^{\text{k,S,low}} - t_1^{\text{k,S,low}}|$. This is equivalent to the following minimization:
\begin{equation}\label{eqn:low_thresh}
    \min_{a_t^k\in \mathcal{A}^k_t} \left|\mathcal{F}_0^{-1}\left(1-\frac{\bar{c}_t^{0,k}+a_tc_k}{s_t}\right) - \mathcal{F}_1^{-1}\left(1-\frac{\bar{c}_t^{1,k}+(1-a_t)c_k}{1-s_t}\right)\right|
\end{equation}
where $\bar{c}_t^{0,k} = \sum_{j=1}^{k-1}\pi_t^jc_j$ and $\bar{c}_t^{1,k} = \sum_{j=1}^{k-1}(1-\pi_t^j)c_j$. For $k=1$, the action that maximizes the score-based reward is equivalent to minimizing $\left|\mathcal{F}_0^{-1}\left(1-\frac{a_tc_1}{s_t}\right) - \mathcal{F}_1^{-1}\left(1-\frac{(1-a_t)c_1}{1-s_t}\right)\right|$.
\end{proof}

\section{Proof details for applicant pool convergence under MFG policy}
\label{sec:app_weighted_MFG_expr}
We first provide a proof of Lemma~\ref{lem:weighted_policy} below. 
\begin{proof}
Firstly, we observe that the optimal policy of institution $k$ can be expressed as a convex combination of its optimal score-based reward action and optimal fair-only action as $\pi^{k}_{t} = \omega_{k,t} a^{k}_{S,t} + (1-\omega_{k,t}) a^{k}_{F,t}$, where $\omega_{k,t}\in[0,1]$. This follows from the fact that the fairness-aware utility function is a concave function of the action, and the score-based reward function and the fairness loss functions are monotone. Then, we claim that the optimal policy of institution $k$ can be expressed as a convex combination of its pre-projection optimal score-based reward action and pre-projection optimal fairness-aware action (before being projected to the feasible action space $\mathcal{A}^k_t$), as $\pi^{k}_{t} = \gamma_{k,t} \left(s_t + \frac{1}{c_k}\left( \sum_{j = 1}^{k-1}c_j (s_t - \pi^{j}_t )\right)\right) + (1-\gamma_{k,t}) \alpha$, where $\gamma_{k,t}\in[0,1]$.

Let $\bar{a}^{k}_{S,t}=\left(s_t + \frac{1}{c_k}\left(\sum_{j = 1}^{k-1}c_j (s_t - \pi^{j}_t )\right)\right)$ and $\bar{a}^{k}_{F,t}=\alpha$, which represent the pre-projected optimal actions. If both $\bar{a}^{k}_{S,t} \in \mathcal{A}^k_t$ and $\bar{a}^{k}_{F,t} \in \mathcal{A}^k_t$, then the claim is straight-forward. Similarly, if only one of them is not in the feasible action space, then the claim is straightforward. The only case that needs to be considered is when both pre-projected optimal actions are not in the feasible action space. If one of them lies on the left of the feasible action interval and the other lies on the right of the feasible action interval, then the claim is straightforward. We will show that pre-projected optimal actions cannot simultaneously lie on the left/right of the feasible action interval. The first case is when both pre-projected optimal actions are on the left of the feasible action interval. Suppose that $\bar{a}^{k}_{F,t} <\max(0,1-\frac{1-s_t-\sum_{j=1}^{k-1}(1-\pi^{j}_t)c_j}{c_k})$ and $\bar{a}^{k}_{S,t}<\max(0, 1-\frac{1-s_t-\sum_{j=1}^{k-1}(1-\pi^{j}_t)c_j}{c_k})$. If the maximum is equal to 0, then there is a contradiction since $\bar{a}^{k}_{F,t} \in[0,1]$. Similarly, if the maximum is equal to $1-\frac{1-s_t-\sum_{j=1}^{k-1}(1-\pi^{j}_t)c_j}{c_k}$, then there is a contradiction since $\bar{a}^{k}_{S,t}> 1-\frac{1-s_t-\sum_{j=1}^{k-1}(1-\pi^{j}_t)c_j}{c_k}$, due to the condition that $\sum_{j=1}^{K}c_j<1$. Thus, pre-projected optimal actions cannot lie on the left of the feasible action interval at the same time. The second case is when both pre-projected optimal actions are on the right of the feasible action interval. Suppose that $\bar{a}^{k}_{F,t} >\min(1, \frac{s_t-\sum_{j=1}^{k-1}\pi^{j}_tc_j}{c_k})$ and $\bar{a}^{k}_{S,t}>\min(1, \frac{s_t-\sum_{j=1}^{k-1}\pi^{j}_tc_j}{c_k})$. If the minimum is equal to 1, then there is a contradiction since $\bar{a}^{k}_{F,t} \in[0,1]$. Similarly, if the minimum is equal to $\frac{s_t-\sum_{j=1}^{k-1}\pi^{j}_tc_j}{c_k}$, then there is a contradiction since $\bar{a}^{k}_{S,t}< \frac{s_t-\sum_{j=1}^{k-1}\pi^{j}_tc_j}{c_k}$, due to the condition that $\sum_{j=1}^{K}c_j<1$. Thus, the pre-projected optimal actions cannot lie on the right of the feasible action interval at the same time. Therefore, our claim holds.

By utilizing this relation, equation \eqref{eqn:opt_score_action} can be expressed as the following by iteratively writing expressions for the optimal score-based reward action in the increasing order of the institutions:
\begin{align}
\scalebox{0.90}{$
       a^{k}_{S,t} =\left[\frac{s_t}{c_k}\left(c_{k}+\displaystyle\sum_{j = 1}^{k-1}c_j {\displaystyle \prod_{i = j}^{k-1}} (1 - \gamma_{i,t} ) \right) - \frac{\alpha}{c_k}\left(\displaystyle\sum_{j = 1}^{k-1}c_j {\displaystyle \prod_{i = j}^{k-1}} (1 - \gamma_{i,t} )\right)\right]_{\mathcal{A}^k_t}
       =\left[s_t + \frac{(s_t-\alpha)}{c_k}\left(\displaystyle\sum_{j = 1}^{k-1}c_j {\displaystyle \prod_{i = j}^{k-1}} (1 - \gamma_{i,t} )\right)\right]_{\mathcal{A}^k_t}
       \label{eqn:opt_score_action_v2}
       $}
\end{align}
We use the above expression and now prove the lemma. The optimal action of the institution $k$ can be expressed as follows:
\begin{equation}
    \pi^{k}_t = \gamma_{k,t} \left(s_t + \frac{(s_t-\alpha)}{c_{k}}\left(\sum_{j = 1}^{k-1}c_j {\displaystyle \prod_{i = j}^{k-1}} (1 - \gamma_{i,t} )\right)  \right)      + (1-\gamma_{k,t}) \alpha 
\end{equation}
We can compute the weighted action as follows;
\begin{equation}
    c_{k}\pi^{k}_t = c_{k} \gamma_{k,t} s_t + \gamma_{k,t}(s_t-\alpha)\sum_{j = 1}^{k-1}c_j {\displaystyle \prod_{i = j}^{k-1}} (1 - \gamma_{i,t} )
    + c_{k}(1-\gamma_{k,t}) \alpha
\end{equation}
We can add and subtract $c_{k}s_t $;
\begin{equation}
    c_{k}\pi^{k}_t = c_{k}s_t  
    + c_{k}(1-\gamma_{k,t})(\alpha-s_t) - \gamma_{k,t}(\alpha-s_t)\sum_{j = 1}^{k-1}c_j {\displaystyle \prod_{i = j}^{k-1}} (1 - \gamma_{i,t} )
    \label{eqn:lem_opt_k}
\end{equation}
We can add and subtract $(\alpha-s_t)\displaystyle\sum_{j = 1}^{k-1}c_j {\displaystyle \prod_{i = j}^{k-1}} (1 - \gamma_{i,t} )$;
\begin{equation}
    c_{k}\pi^{k}_t = c_{k}s_t  
    + c_{k}(1-\gamma_{k,t})(\alpha-s_t) + (1-\gamma_{k,t})(\alpha-s_t)\sum_{j = 1}^{k-1}c_j {\displaystyle \prod_{i = j}^{k-1}} (1 - \gamma_{i,t} ) - (\alpha-s_t)\displaystyle\sum_{j = 1}^{k-1}c_j {\displaystyle \prod_{i = j}^{k-1}} (1 - \gamma_{i,t} )
\end{equation}
It can be grouped as follows:
\begin{equation}
    c_{k}\pi^{k}_t = c_{k}s_t  
    + (\alpha-s_t)\sum_{j = 1}^{k}c_j {\displaystyle \prod_{i = j}^{k}} (1 - \gamma_{i,t} ) - (\alpha-s_t)\displaystyle\sum_{j = 1}^{k-1}c_j {\displaystyle \prod_{i = j}^{k-1}} (1 - \gamma_{i,t} )
\end{equation}
It can be seen that the second and third terms are telescoping, and the summation of the weighted actions of all institutions can be written as follows:
\begin{equation}
    \displaystyle \sum_{k = 1}^{K} c_k \pi^{k}_t = s_t\displaystyle \sum_{k = 1}^{K} c_k +  (\alpha-s_t) \displaystyle\sum_{k = 1}^{K} c_k \displaystyle \prod_{i = k}^{K} (1 - \gamma_{i,t} )
\end{equation}
After dividing by the total capacity, we obtain the desired expression.
\end{proof}

We will now utilize the result of the above lemma in showing the proof of Theorem~\ref{thm:conv_mfg}, where we state that the applicant pool proportion, and the admission proportions of the institutions converge to the long-term fairness target set by the agents, under the MFG policy with pure positive reinforcement model.
\begin{proof}
Suppose that the state of the MDP equals the fairness target $\alpha$. It follows that the top institution maximizes its utility by selecting $\alpha$ proportion of group $0$ among all its admitted applicants $\pi_t^{1} = \alpha$, as its optimal score-based reward action itself is $a^{1}_{S,t} = s_t = \alpha$. Further, $a^{2}_{S,t} = \alpha$ by equation~\eqref{eqn:opt_score_action}, which implies that $\pi_t^{2} = \alpha$. It can be seen that $\pi_t^{k} = \alpha, \forall k$, due to which the weighted average $\pi^W_t = \alpha$, i.e., it is a fixed-point of the weighted policy. 
We will show the uniqueness of the fixed point by showing $\gamma_{i,t}\neq1, \forall i \in [K]$. Assume that $s_t<\alpha$. Then, the optimal score-based reward action of the top institution is $a^{1}_{S,t} = s_t$, since the top institution maximizes its score-based reward when it sets same threshold for both groups. The optimal fair-only action is $a^{1}_{F,t} = [\alpha]_{\mathcal{A}^1_t}$. Since $a^{1}_{F,t}>s_t$ the MFG policy for the top institution will accept more from the minority group than the applicant proportion due to the structure of the weighted MFG policy. Then, $\pi^{1}_{t}>s_t$ and $\gamma_{1,t}\neq1$. The score-based reward of the second institution is maximized when it sets the same threshold for both groups, this implies that $a^{2}_{S,t} < s_t$ because the first institution admitted more from the minority group than the applicant proportion. The optimal fair-only action is $a^{2}_{F,t} = [\alpha]_{\mathcal{A}^2_t}$. Then, $\pi^{2}_{t}>s_t$ and $\gamma_{2,t}\neq1$. The same argument can be made for the subsequent institutions, since the optimal score-based reward action of the institution $k$ is $a^{k}_{S,t} < s_t$ because the previous institutions admitted more from the minority group than the applicant proportion. The optimal fair-only action is $a^{k}_{F,t} = [\alpha]_{\mathcal{A}^k_t}$. Then, $\pi^{k}_{t}>s_t$ and $\gamma_{k,t}\neq1$. The same argument can be used when $s_t>\alpha$.

Next, since from the update for pure positive reinforcement \eqref{eqn:pure_pos_reinf}, it follows that the mean parameter $\theta_{t}$ drifts towards the fairness target $\alpha$, irrespective of the state, due to the structure of the weighted MFG policy. 
In addition, since $\alpha$ is a unique fixed-point, as the step size is decaying with time, it can be shown that the mean parameter converges to the fairness target. Let $d_t=\frac{1}{2}(\theta_t-\alpha)^2$. Fix an $\epsilon>0$. Then, we need to show that there exists some $t_0(\epsilon)$ such that for all $t>t_0(\epsilon)$,
\begin{align}
    d_{t+1} &\leq d_t - \zeta_t, \text{ if } d_t\geq\epsilon \label{eq:thm21} \\ 
    d_{t+1} &\leq c\epsilon \text{ if } d_t<\epsilon \label{eq:thm22}
\end{align}
where $c$ is a positive constant. Moreover $\zeta_t>0$ and $\sum_{t=0}^{\infty}\zeta_t=\infty$. If the above conditions are satisfied, then eventually for some $t_1(\epsilon)\geq t_0(\epsilon)$, we have $d_t<\epsilon$. But due to equations~\eqref{eq:thm21} and~\eqref{eq:thm22}, we have $d_{t+1}\leq c\epsilon$ for all $t>t_1(\epsilon)$. Since $\epsilon$ is arbitrary, we have $\theta_t\rightarrow\alpha$ as $t\rightarrow\infty$.
\begin{align}
    d_{t+1} &= \frac{1}{2}(\theta_{t+1}-\alpha)^2 \nonumber \\
    &= \frac{1}{2}(\left[ \theta_t + \eta_t (\pi^{W}_{t} - s_t) \right]_{\mathcal{C}}-\alpha)^2 \nonumber \\
    &\leq \frac{1}{2}(\theta_t + \eta_t (\pi^{W}_{t} - s_t)-\alpha)^2 \nonumber \\
    &= d_t + \eta_t(\theta_t-\alpha)(\pi^{W}_{t}-s_t) + \frac{1}{2}\eta_t^2(\pi^{W}_{t} - s_t)^2 \nonumber \\
    &\leq d_t + \eta_t(\theta_t-\alpha)(\pi^{W}_{t}-s_t) + \frac{1}{2}\eta_t^2\label{eq:thm23a} \\
    &\leq d_t +\frac{\eta_t}{2}((\theta_t-\alpha)^2+1) + \frac{1}{2}\eta_t^2 \label{eq:thm23}
\end{align}
Since $\eta_t\rightarrow0$, if $d_t<\epsilon$, then $d_{t+1}<c\epsilon$ for some $c>0$.

When $d_t\geq\epsilon$, first we will account for the stochasticity of $s_t$. We have $\pi^{W}_{t}-s_t = \pi^{W}_{t}-\theta_t+(\theta_t-s_t)$. Denoting $z_t=\theta_t-s_t$, using equation~\eqref{eq:thm23a}, we have
\begin{equation}
    d_{t+1}\leq d_t + \eta_t(\theta_t-\alpha)(\pi^{W}_{t}-\theta_t+z_t) + \frac{1}{2}\eta_t^2\nonumber
\end{equation}
where $z_t$ is a zero-mean random variable. Also $\mathbb{E}[z_t^2]=var(s_t)<\infty$, which is bounded. Therefore, $\nu_t:=\sum_{i=0}^{t}\eta_tz_i$ is a martingale and $\mathbb{E}[\nu_t^2]$ is also bounded. This implies by the martingale convergence theorem that $\nu_t$ converges to a finite random variable. Therefore, we have $\sum_{i=t}^{\infty}\eta_tz_i\rightarrow0$. Since, $|\theta_t-\alpha|$ is bounded, the effect of the noise $z_t$ is negligible. Therefore we have
\begin{equation}
    d_{t+1}\leq d_t + \eta_t(\theta_t-\alpha)(\pi^{W}_{t}-\theta_t) + \frac{1}{2}\eta_t^2\nonumber
\end{equation}
We want to show that 
\begin{equation}
    (\theta_t-\alpha)(\pi^{W}_{t}-\theta_t)\leq-\delta(\epsilon) \label{eq:thm24}
\end{equation}
for some $\delta(\epsilon)>0$. If this holds, we have
\begin{equation}
    d_{t+1} \leq d_t - \eta_t\delta(\epsilon) + \frac{\eta_t^2}{2} \label{eq:thm25}
\end{equation}
Let us denote $\zeta_t=\eta_t\delta(\epsilon) - \frac{\eta_t^2}{2}$. Since, $\eta_t\rightarrow0$, there exists $t_2(\epsilon)$ such that $\zeta_t>0$ for all $t>t_2(\epsilon)$. Moreover, due to assumptions on the step size, we have $\sum_{t=0}^{\infty}\zeta_t=\infty$.

What remains to show is equation~\eqref{eq:thm24}. In the regime where the number of applicants $N$ is large, we can see that the state $s_t$ is equal to its mean $\theta_t$ with probability approaching $1$ through Chebyshev inequality. When $d_t\geq\epsilon$, since $s_t$ is equal to its mean $\theta_t$, we need to consider only cases (i) $s_t<\alpha$ and (ii) $s_t>\alpha$. Under both cases, we have $(\theta_t-\alpha)(\pi^{W}_{t}-\theta_t)<0$ due to the structure of the weighted MFG policy when the score distributions are identical.
\end{proof}

\section{Negative feedback in role model reinforcement}
\label{sec:app_role_model_neg_fb}
Here, we provide a proof for Proposition~\ref{prop:role_model}, and show that MFG policy can cause negative feedback under role model reinforcement and drive the under-represented group out of the system.
\begin{proof}
We assume that the MFG policy assigns non-zero weight to the long-term fairness objective, i.e., $\lambda> 0$. Since we are interested in the regime where the number of applicants $N_t$ is large, we assume that the histograms of the scores of applicants approach the distribution. Without loss of generality, we denote group $0$ to be the under-represented group, and assume that the initial state is less than the fairness target, $s_t < \alpha$. However, the same procedure applies when group $0$ has a higher proportion than the long-term fairness target. We will also assume that the initial mean parameter $\theta_t$ is small. 

We recall that for the top institution, the action maximizing its score-based reward  $a^{1}_{S,t} = s_t$, the current state, by Theorem~\ref{thm:opt_score_act} and the action minimizing its fairness loss $a^{1}_{F,t} = [\alpha]_{\mathcal{A}^1_t}$, the long-term fairness target. We also know the MFG policy for the top institution will be the convex combination $\pi^{1}_{t} = \gamma_1 s_t + (1-\gamma_k)\alpha$. Then, $\pi^{1}_{t}>s_t$ because $\alpha>s_t$ by our assumption. Viewing the first institution's MFG policy as applying group-dependent lower and upper thresholds on the score distributions as in the proof of Theorem~\ref{thm:opt_score_act}, and admitting all applicants between the group-specific thresholds, we can infer that the lower threshold for group 0 (denoted by $t_0^{\text{1, low}}$) is strictly less than the lower threshold for group 1 (denoted by $t_1^{\text{1, low}}$), due to the fact that $\pi^{1}$ admits more from the minority group than the applicant proportion. Implicitly, the upper thresholds for the first institution are infinity. Thus we have:
\begin{equation}
    t_{0}^{\text{1, low}} < t_1^{\text{1, low}}
\end{equation}
Note that these thresholds are different from the thresholds optimizing only for the score-based rewards, which have also been described by the same notation in the proof of Theorem~\ref{thm:opt_score_act}. 

Let us assume that the fraction of the role models $r = \epsilon$. The evolution of the applicant pool is such that only those admitted applicants with scores larger than a certain threshold determined by parameter $r$ will contribute to reinforcing the pool. Let us denote this threshold for institution $k$ as $t^{k}_{r}$. As the parameter $r$ decreases, the threshold $t^{k}_{r}$ increases. It follows that there exists $r$ small enough, such that the role model threshold for the first institution is $t^{1}_{r}\geq t_1^{\text{1, low}}$. We remark that the role model threshold is independent of the group membership.

Now, note that $r^1_t$ is equivalent to the ratio of the number of group $0$ applicants with scores higher than $t^{1}_{r}$, to the number of applicants with scores larger than $t^{1}_{r}$. Thus we can write an expression for $r^1_t$ in terms of the CDF of the score distribution, denoted by $\mathcal{F}$, as
\begin{equation}
    r_t^1=\frac{s_t\left(1-\mathcal{F}(t_r^{1})\right)}{s_t\left(1-\mathcal{F}(t_r^{1})\right)+(1-s_t)\left(1-\mathcal{F}(t_r^{1})\right)}=s_t.
    \label{eqn:role1_state}
\end{equation}

For the subsequent institutions, it is known from \eqref{eqn:opt_score_action_v2} that  for all $k\geq 2$, if $s_t < \alpha$, the action optimizing the score-based reward is $a^{k}_{S,t} < s_t$.
Therefore, the policy for the $k^{th}$ institution ($k\geq 2$) is
\begin{equation}
\pi^{k}_{t} = (1-\gamma_{k,t}) \alpha+\gamma_{k,t} \left(s_t + \frac{1}{c_k}\left( \sum_{j = 1}^{k-1}c_j (s_t - \pi^{j}_t )\right)\right) > a^{k}_{S,t}\label{eqn:pol_more_than_a_star}
\end{equation}
since we have $\alpha>s_t>a^{k}_{S,t}$. Note that for the MFG policies of the institutions, the lower threshold of institution $k-1$ will be equal to the upper threshold of institution $k$, as it would admit all applicants with scores in between the lower and upper thresholds. 

If the institutions with $k\geq 2$ were optimizing only the score-based reward their lower thresholds would be equal across the groups, as argued in the equation \eqref{eqn:low_thresh}. But here, since the MFG policy admits a higher proportion from group $0$, i.e., \eqref{eqn:pol_more_than_a_star}, the thresholds are such that the lower threshold of group 0 is always less than the lower threshold of group 1, $\forall k \in [K]$
\begin{equation}
    \label{eqn:low_k}
    t_{0}^{\text{k,low}} < t_1^{\text{k,low}}.
\end{equation}
Furthermore, if $r$ is small enough, the role model thresholds for all subsequent institutions are large enough to satisfy $t_1^{\text{k,low}} \leq t_r^{k}.$
Then, we can express the proportion of the role models from group 0 for the institution $k$ as
\begin{equation}
    r_t^k=\frac{s_t\left(\mathcal{F}(t_0^{\text{k,up}})-\mathcal{F}(t_r^{k})\right)}{s_t\left(\mathcal{F}(t_0^{\text{k,up}})-\mathcal{F}(t_r^{k})\right)+(1-s_t)\left(\mathcal{F}(t_1^{\text{k,up}})-\mathcal{F}(t_r^{k})\right)}.
\end{equation}
Using $t_1^{\text{k,up}}>t_0^{\text{k,up}}$, we can upper bound $r_t^k$ as 
\begin{align}    &r_t^k<\frac{s_t\left(\mathcal{F}(t_0^{\text{k,up}})-\mathcal{F}(t_r^{k})\right)}{s_t\left(\mathcal{F}(t_0^{\text{k,up}})-\mathcal{F}(t_r^{k})\right)+(1-s_t)\left(\mathcal{F}(t_0^{\text{k,up}})-\mathcal{F}(t_r^{k})\right)}\\ &\implies r_t^k<s_t    \label{eqn:role_less_than_state} 
\end{align}
for all $k\in\{2,3,\dots,K\}$ and $r_t^1=s_t$.

Under the role model reinforcement, group $0$'s applicant proportion receives a positive drift only if the weighted parameter $\pi^r_t$ in equation \eqref{eqn:role_pol} is larger than $s_t$. Due to \eqref{eqn:role1_state} and \eqref{eqn:role_less_than_state}, the weighted proportion of the role models is
\begin{equation}
    \pi^r_t = \frac{\sum_{k = 1}^{K} c_k r_t^k}{\sum_{k = 1}^{K} c_k}<s_t \implies \pi^r_t-s_t<0. 
\end{equation}
Hence, the pool update parameter $\theta_{t+1} = \theta_t + \eta_t (\pi^{r}_{t} - s_t) < \theta_t$. If the initial mean parameter $\theta_t$ is small enough, i.e., group $0$ is heavily under-represented in the initial pool, with a large probability the future state $s_{t+1} < \alpha$. Hence we can {\it approximately} see that the mean parameter of the proportion of minority group in the applicant pool decreases to zero. Therefore, we show that MFG policy can cause a negative feedback loop under role model reinforcement if the role model proportion is small enough, driving the under-represented group out of the system.
\end{proof}

\section{Additional experimental results}
\label{app:exps}
\subsection{Identical score distributions}
\label{app:iden_score_dists}
\begin{figure}[t]
    \centering
    \begin{subfigure}[t]{0.32\columnwidth}
         \centering
         \includegraphics[width=\textwidth]{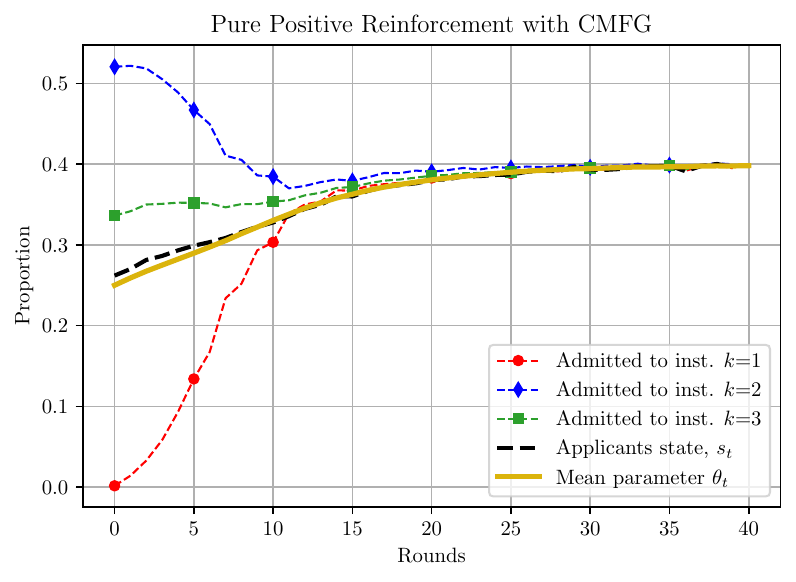}
         \caption{}
         \label{fig:pure_pos_cmfg}
     \end{subfigure}
     \hfill
     \begin{subfigure}[t]{0.32\columnwidth}
         \centering
         \includegraphics[width=\textwidth]{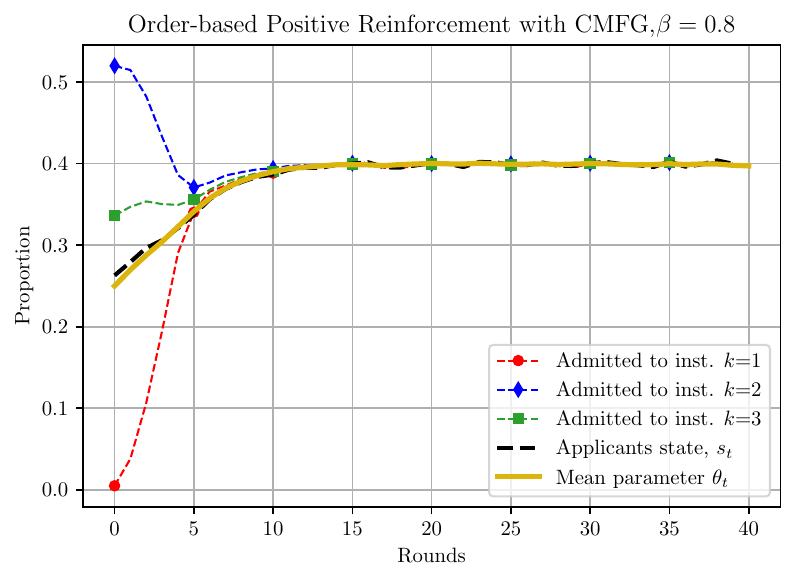}
         \caption{}
         \label{fig:order_cmfg}
     \end{subfigure}
     \hfill
     \begin{subfigure}[t]{0.32\columnwidth}
         \centering
         \includegraphics[width=\textwidth]{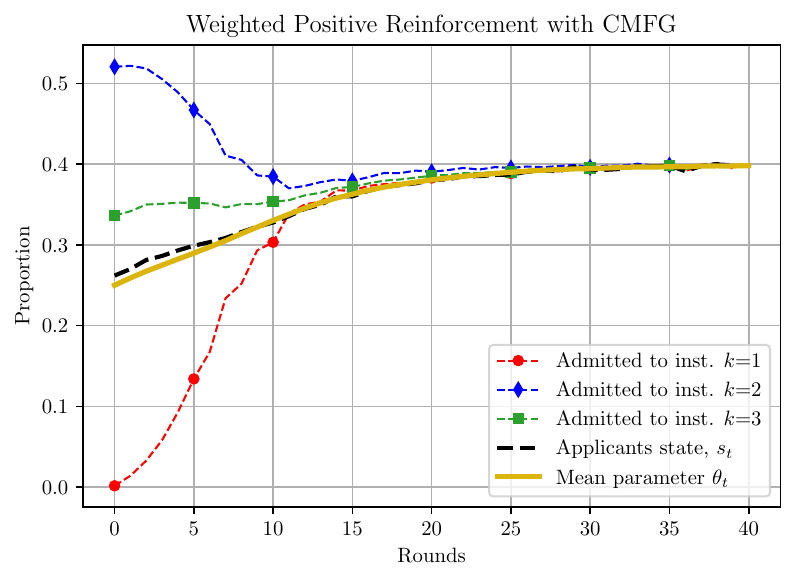}
         \caption{}
         \label{fig:weight_cmfg}
     \end{subfigure}
    \caption{ The centralized MFG policy achieves long-term fairness under pure positive \subref{fig:pure_pos_cmfg}, order-based \subref{fig:order_cmfg} and the weighted positive \subref{fig:weight_cmfg} reinforcement models.}
\end{figure}
We show the evolution of applicant and admission proportions of the CMFG policy under the pure positive reinforcement model, order-based and weighted positive reinforcement models in Figures~\ref{fig:pure_pos_cmfg},~\ref{fig:order_cmfg} and~\ref{fig:weight_cmfg}. The parameter settings are exactly the same as in Section~\ref{sec:exps_synt} for the case when the score distributions are identical. In these figures, we can observe that an institution's admission proportion is larger than the applicant proportion, and since the weighted average of the admission proportion governs the evolution, the pool gets positive feedback and approaches the fairness target.

\subsection{Distinct score distributions}
\label{sec:app_dist_dists}
Here, we consider the role model reinforcement with $r = 0.5$, under distinct score distributions, with all other parameters as described in Section~\ref{sec:exps_synt}. As seen in Figure~\ref{fig:8a}, the MFG policy again leads to a negative feedback loop with the distinct score distributions, resulting in the reduction of the minority group in the applicants pool. As seen in Figure~\ref{fig:8b}, the second and the third institutions have a very small fraction of role models as a result of the first institution admitting the top minority applicants. 
However, the CMFG policy could alleviate this, resulting in an equilibrium point lower than the initial mean parameter as shown in Figure~\ref{fig:8c}. In Figure~\ref{fig:8d}, it is evident that all institutions feature role models from the minority group, leading to an equilibrium point within the applicant pool.
\begin{figure}[t]
\centering
  \begin{subfigure}[t]{0.24\columnwidth}
         \centering
         \includegraphics[width=\textwidth]{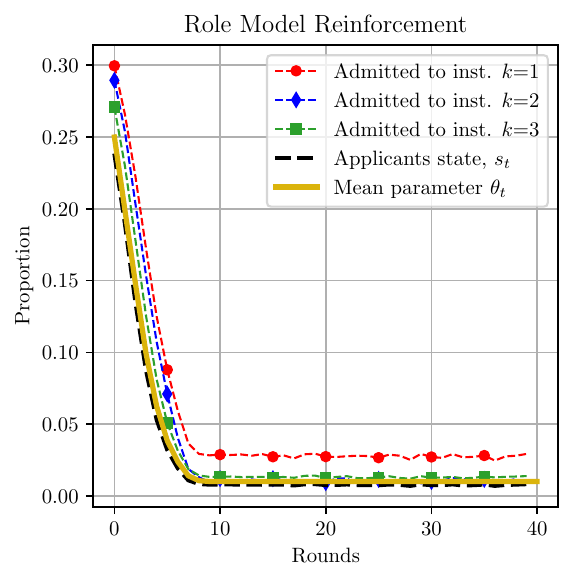}
         \caption{}
         \label{fig:8a}
     \end{subfigure}
  \begin{subfigure}[t]{0.24\columnwidth}
         \centering
         \includegraphics[width=\textwidth]{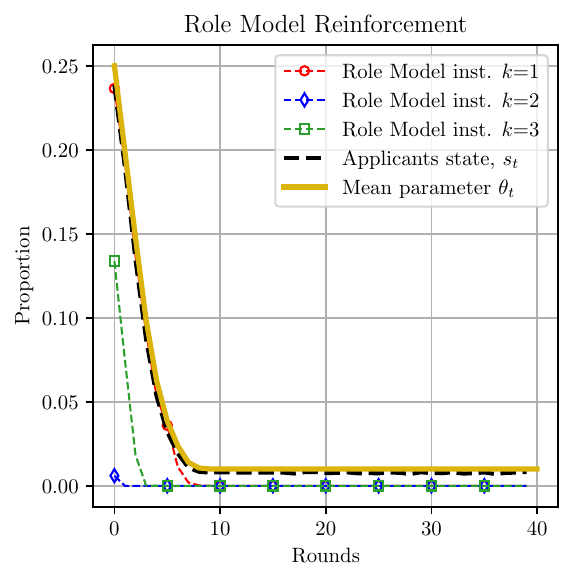}
         \caption{}
         \label{fig:8b}
     \end{subfigure}  \begin{subfigure}[t]{0.24\columnwidth}
         \centering
         \includegraphics[width=\textwidth]{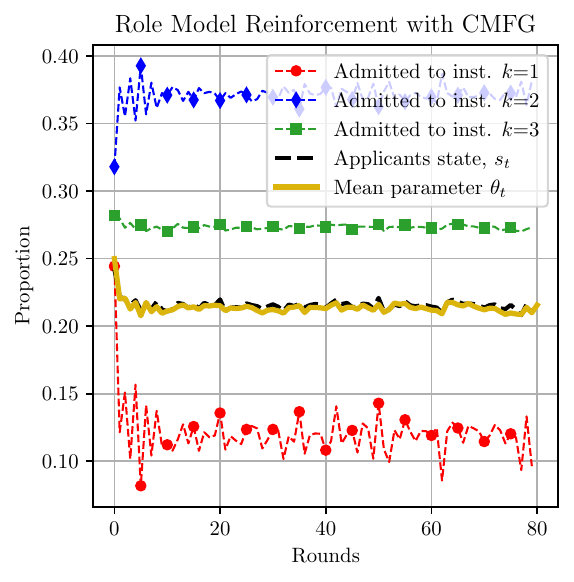}
         \caption{}
         \label{fig:8c}
     \end{subfigure}
     \begin{subfigure}[t]{0.24\columnwidth}
         \centering
         \includegraphics[width=\textwidth]{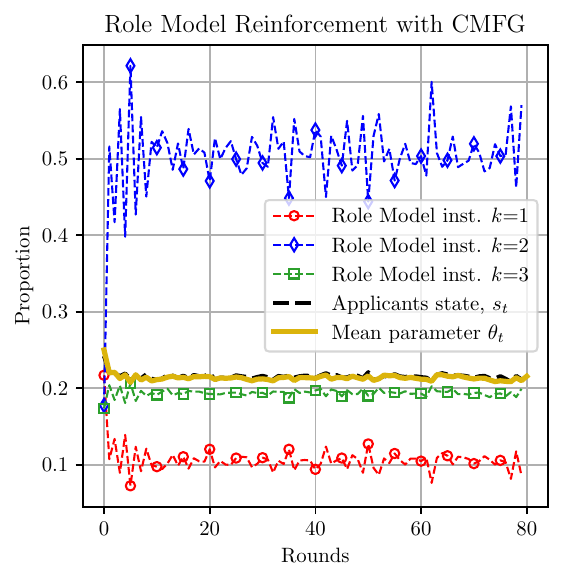}
         \caption{}
         \label{fig:8d}
     \end{subfigure}
    \caption{~\subref{fig:8a} MFG policy creates negative feedback under role-model reinforcement, under distinct score distributions.~\subref{fig:8b} The evolution of the proportions of role models for each institution, under MFG policy.~\subref{fig:8c} CMFG policy avoids negative feedback under distinct score distributions.~\subref{fig:8d} The evolution of the proportions of role models for each institution, under CMFG policy.
}
\label{fig:synt_4}
\end{figure}

\subsection{Semi-synthetic dataset}
\label{app:role_model_lsa_dataset}
\begin{figure}[t]
\centering
  \begin{subfigure}[t]{0.32\columnwidth}
         \centering
         \includegraphics[width=\textwidth]{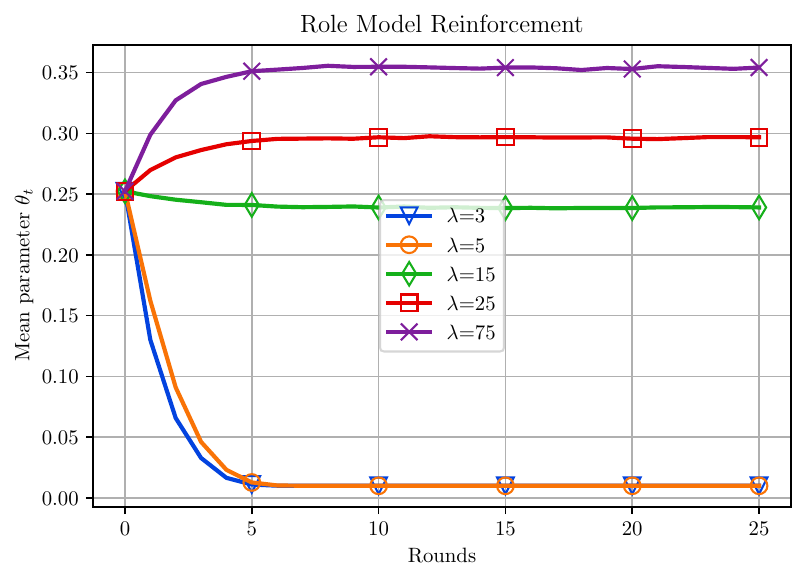}
         \caption{}
         \label{fig:9a}
     \end{subfigure}
     \hspace{1cm}
     \begin{subfigure}[t]{0.32\columnwidth}
         \centering
         \includegraphics[width=\textwidth]{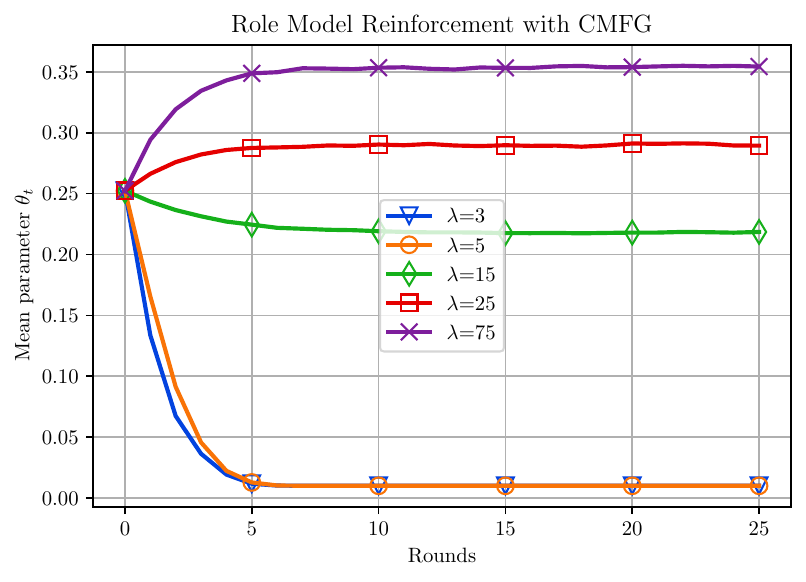}
         \caption{}
         \label{fig:9b}
     \end{subfigure}
    \caption{~\subref{fig:9a},~\subref{fig:9b} evolution of the mean parameter, $\theta_t$, and its equilibrium with varying values of $\lambda$, under both the MFG and CMFG policies, under role-model reinforcement.
}
    \label{fig:law_2}
\end{figure}
We consider the role-model reinforcement with $r=0.8$, where the top $80\%$ of the admitted applicants are considered role models for the resource pool. Figures~\ref{fig:9a} and~\ref{fig:9b} show the mean parameter $\theta_t$ at equilibrium for the MFG and CMFG policies, respectively, with varying values of $\lambda$. Both policies perform similarly on the law school dataset, possibly due to a larger difference between the score distributions of the groups.

\end{document}